\documentclass[11pt,a4paper]{scrartcl}
\usepackage[top=3cm, bottom=3cm, left=2cm, right=2cm]{geometry}

\usepackage{booktabs}

\usepackage{color}
\usepackage{latexsym}              
\usepackage{amsmath}               
\usepackage{amssymb}               
\usepackage{amsfonts}              
\usepackage{amsthm}                
\usepackage{mathtools}
\usepackage{multirow}
\usepackage{tikz}                  
\usetikzlibrary{arrows,positioning,shapes}

\usepackage[mathcal]{eucal}

\usepackage{cleveref}
\crefname{assumption}{Assumption}{Assumptions}
\crefname{equation}{Eq.}{Eqs.}
\crefname{figure}{Fig.}{Figs.}
\crefname{table}{Table}{Tables}
\crefname{section}{Sec.}{Secs.}
\crefname{theorem}{Thm.}{Thms.}
\crefname{lemma}{Lemma}{Lemmas}
\crefname{corollary}{Cor.}{Cors.}
\crefname{example}{Example}{Examples}
\crefname{appendix}{Appendix}{Appendixes}
\crefname{remark}{Remark}{Remark}

\makeatother

\newcounter{remark}[section]

\newcommand{\calX}{\mathcal{X}}
\newcommand{\calY}{\mathcal{Y}}
\newcommand{\calZ}{\mathcal{Z}}
\newcommand{\calE}{\mathcal{E}}

\newcommand{\calP}{\mathcal{P}}
\newcommand{\calH}{\mathcal{H}}
\newcommand{\calO}{\mathcal{O}}

\newcommand{\calG}{\mathcal{G}}

\newcommand{\calC}{\mathcal{C}}

\newcommand{\calR}{\mathcal{R}}

\newcommand{\ERL}{\calE}

\DeclareMathAlphabet{\mathsfsl}{OT1}{cmss}{m}{sl}

\renewcommand{\phi}{\varphi}

\newcommand{\Rspace}[1]{\mathbb{R}^{#1}}

\newcommand{\R}{\mathbb{R}}
\newcommand{\N}{\mathbb{N}}

\newcommand{\fstar}{f^\star}

\newcommand{\eqal}[1]{
\begin{align}
#1
\end{align}
}

\newcommand{\eqals}[1]{
\begin{align*}
#1
\end{align*}
}

\newcommand{\argmin}{\operatorname*{arg\; min}}
\newcommand{\argmax}{\operatorname*{arg\; max}}

\newcommand{\Expect}{\operatorname{\mathbb{E}}}

\newcommand{\bigO}{\mathcal{O}}

\theoremstyle{plain}  
\newtheorem{theorem}{Theorem}[section]
\newtheorem{defn}[theorem]{Definition}

\newtheorem{lemma}[theorem]{Lemma}

\newtheorem{corollary}[theorem]{Corollary}
\newtheorem{remark}[theorem]{Remark}

\newtheorem{example}[theorem]{Example}

\newcommand{\cardY}{|\calY|}
\newcommand{\cardZ}{|\calZ|}

\begin{document}

%

%
\title{Sharp Analysis of Learning with Discrete Losses}

\author{\textbf{Alex Nowak-Vila, Francis Bach, Alessandro Rudi} \\ [2ex]
INRIA - D\'epartement d'Informatique de l'\'Ecole Normale Sup\'erieure \\
PSL Research University \\
 Paris, France \\\\
}

\maketitle
\begin{abstract}
The problem of devising learning strategies for discrete losses (e.g., multilabeling, ranking) is currently addressed with methods and theoretical analyses {\em ad-hoc} for each loss.

In this paper we study a least-squares framework to systematically design learning algorithms for discrete losses, with quantitative characterizations in terms of statistical and computational complexity. In particular we improve existing results by providing explicit dependence on the number of labels for a wide class of losses and faster learning rates in conditions of low-noise.

Theoretical results are complemented with experiments on real datasets, showing the effectiveness of the proposed general approach.
\end{abstract}


\section{Introduction}

Structured prediction with discrete labels of high cardinality is ubiquitous in machine learning, e.g., in multiclass problems, multilabel learning, ranking, ordinal regression, etc. \cite{bakir2007predicting, crammer2001algorithmic,read2011classifier, pedregosa2017consistency}.

These supervised learning problems typically come with computational and theoretical challenges:
{\em
\begin{itemize}
    \item[(1)] how to design efficient algorithms dealing with potentially large number of data and labels?
    \item[(2)]  even if learning is computationally feasible, how to make sure that the resulting algorithm leads to improved accuracy {\em on the test set}?
\end{itemize}}
Many special cases are often addressed in an {\em ad-hoc} fashion in terms of consistency, algorithms and convergence rates, depending on the specific loss used in each application to quantify the performance of predictors.
    
A few generic learning frameworks exist: 
(a) conditional random fields~\cite{Lafferty:2001:CRF:645530.655813,settles2004biomedical} use conditional probabilistic modelling typically combined with maximum likelihood estimation, but may lead to intractable probabilistic inference and cannot easily incorporate structured losses which are needed in applications \cite{volkovs2011loss};
(b) Structured SVM~\cite{tsochantaridis2004support,joachims2006training} extended the class of problems where a systematic max-margin framework can be applied, with the incorporation of arbitrary losses, but they are not consistent in general, that is, even with infinite amounts of data, they would not lead to optimal predictions \cite{tewari2007consistency};
(c) more recently, least-squares (or quadratic surrogate) frameworks \cite{ciliberto2016consistent,osokin2017structured} have emerged. Such approaches can tackle arbitrary discrete losses producing consistent estimators and have the potential to provide a systematic way to design learning algorithms with both statistical and computational guarantees. However, no sharp analyses exist yet, quantifying the impact of crucial quantities like the number of labels or the level of noise on the statistical and computational properties of the resulting algorithms. The goal of this paper is to characterize explicitly such impact for a number of widely used loss functions in the context of multilabeling and ranking, showing the effectiveness of least-squares frameworks for structured prediction with discrete labels.

We make the following contributions: 
\begin{itemize}
\item[--] We provide quantitative characterizations of the statistical and computational complexity for the least-squares framework of~\cite{ciliberto2016consistent} depending on the number of labels and the number of examples. The characterization is explicit for a wide family of common losses in ranking and multilabel learning (see \cref{sec:affine,sec:multilabel_and_ranking,sec:comp-considerations}).
\item[--] We propose a margin condition for discrete losses (generalizing the Tsybakov condition for binary classification \cite{tsybakov2004optimal}) and obtain fast learning rates for the framework of \cite{ciliberto2016consistent}, that are adaptive to the proposed condition (see \cref{sec:lownoise}).
\item[--] Our analysis encompasses many previous results on special cases and provides improved learning rates over existing generic structured prediction frameworks (see \cref{sec:related_work}).
\item[--] We conduct a series of experiments highlighting the benefits of the considered least-squares framework on ranking and multilabel problems (\cref{sec:exp}).
\end{itemize}

\section{Background}\label{sec:setting}
The problem of {\em supervised learning} consists in learning from examples the function relating inputs with observations/labels. More specifically, let $\calY$ be the space of observations, denoted {\em observation space} or {\em label space} and $\calX$ be the {\em input space}. The quality of the predicted output is measured by a given {\em loss function} $L$. 
In many scenarios the output of the function is in a different space than the observations (see \cref{sec:multilabel_and_ranking} for some examples). We denote by $\calZ$ the {\em output space}, so
\begin{equation}\label{eq:loss_function}
    L:\calZ\times \calY\longrightarrow \R,
\end{equation}
where $L(z, y)$ measures the cost of predicting $z$ when the
observed value is $y$. 
Finally the data are assumed to be distributed according to a probability measure $P$ on $\calX \times \calY$. The goal of supervised learning is then to recover     the function $\fstar$ minimizing the {\em expected risk} $\ERL(f)$ of the loss,
\eqal{\label{eq:optimal_f}
\fstar = \argmin_{f: \calX \to \calZ}\ERL(f), \quad \ERL(f) =\Expect L(f(X), Y),
}
given only a number of examples $(x_i, y_i)_{i=1}^n$, with $n \in \mathbb{N}$, sampled independently from $P$. 
The quality of an estimator $f$ for $\fstar$ is measured in terms of the {\em excess risk} $\ERL(f) - \ERL(\fstar)$.
%

\subsection{Quadratic Surrogate method}\label{sec:recall-qs}
A systematic way to solve the problem in \cref{eq:optimal_f} is to consider that $\fstar$ is characterized as follows \cite{steinwart2008support,ciliberto2016consistent}:
$$
\fstar(x) = \argmin_{z \in \calZ} ~\ell(z, x),
$$
where $\ell(z, x) = \int_\calY L(z,y) dP(y|x)$ is the 
{\em Bayes risk}, defined as the conditional expectation of $y$ given $x \in\calX$. 
The quadratic surrogate (QS) for structured prediction, introduced in \cite{ciliberto2016consistent}, is a natural estimator that has the following form,
\eqal{\label{eq:inference-0}
\widehat{f}(x) = \argmin_{z \in \calZ}~ \widehat{\ell}(z,x),
}
where $\widehat{\ell}(z,x) := \sum_{i=1}^n {\alpha_i}(x) L(z, y_i)$. Here $(\alpha_i)_{i = 1}^n$ are suitable functions defined explicitly in terms of the observed data (not on $L$) and will be discussed later (see \cref{eq:weights,eq:qs_estimator}). Informally, the closer $\widehat{\ell}(z,x)$ is to $\ell(z,x)$, the closer $\widehat{f}$ will  be  to $\fstar$ in terms of the excess risk. In \cite{ciliberto2016consistent} a detailed statistical framework  analyzes the generalization properties of the derived estimator, that will be recalled in the next paragraph. Here we point out that a crucial aspect of the algorithm in \cref{eq:inference-0}, that makes it appealing from a practical viewpoint, is that we can directly apply it given the loss at hand, without the need to devise a different surrogate (and consequently a different algorithm and theoretical analysis) {\em ad-hoc} for each specific loss. 


\textbf{Statistical properties of Quadratic Surrogate.}
Here we recall some generalization properties of the QS estimator from \cite{ciliberto2016consistent}, that will be extended in \cref{sec:analysis}. First, assume that the loss $L$ is a structure encoding loss function (SELF), i.e., it can be written as,
\begin{equation}\label{eq:self_assumption}
    L(z,y) = \langle\phi(z), V\psi(y)\rangle_{\calH},
\end{equation}
where $\calH$ is a separable Hilbert space with $\left\langle{\cdot},{\cdot}\right\rangle_{\calH}$ the associated inner product, $V:\calH\rightarrow\calH$ is a bounded linear operator and $\phi:\calZ \to \calH$, $\psi: \calY \to \calH$.

Note that by assuming $\calZ,\calY$ discrete and finite, then every loss function on $\calZ,\calY$ is SELF. Indeed \cref{eq:self_assumption} is recovered by setting $\calH=\Rspace{\cardZ}$,
$V=(L(z,y))_{z\in\calZ,y\in\calY}\in\Rspace{\cardZ\times\cardY}$ the loss matrix, and $\phi(z)=e_z, \psi(y) = e_y$ the vectors of the canonical basis in $\Rspace{\cardZ}$ and $\Rspace{\cardY}$, respectively (For the case of continuous $\calZ,\calY$ see \cite{ciliberto2016consistent}). 
%

\begin{table*}[t!]
    \hspace{-0.01\textwidth}
    \resizebox{1.018\textwidth}{!}{
        \begin{tabular}{@{}llllll@{}} \toprule
            \multicolumn{6}{c}{Multilabel and Ranking measures} \\ \cmidrule(r){1-6}
            Measure & $\calZ$ & Definition & $r$ & A  & $\text{INF}_{F}(\cardZ)$ \\ \midrule
            
            0-1 $(\downarrow)$& $\calP_m$ & $1(z\neq y)$ & $2^m$ & $2^{m/2}$ & 
            $\calO(n\wedge 2^m)$\\ \addlinespace[5pt]
            
            Block 0-1 $(\downarrow)$& $\calP_m$ & $1(z\in B_j,y\notin B_j,~j\in[b])$ & $b$ & $\sqrt{b}$ & 
            $\bigO(b)$ \\ \addlinespace[5pt]
            
            Hamming $(\downarrow)$& $\calP_m$ & $\frac{1}{m}\sum_{j=1}^m1([z]_j\neq [y]_j)$ & $m$ & $\frac{1}{2}$ & 
            $\bigO(m)$\\ \addlinespace[5pt]
            
            F-score $(\uparrow)$& $\calP_m$ & $2\frac{|z\cap y|}{|z|+|y|}$ & $m^2+1$ & $\sqrt{2}m$ & 
            $\bigO(m^2)$\\  \addlinespace[5pt]
            
            Prec@k $(\uparrow)$& $\calP_{m,k}$ & $\frac{|z\cap y|}{k}$ & $m$ & $\sqrt{\frac{m}{k}}$ & 
            $\bigO(m\log k)$ \\ \addlinespace[5pt]
            
            NDCG $(\uparrow)$& $\mathfrak{S}_m$ & $\frac{1}{N(r)}\sum_{j=1}^mG([r]_j)D_{\sigma(j)}$ & $m$ &
            $\sqrt{m}~(\sum_{j}D_j^2)^{\frac{1}{2}}G_{\max}$ & $\bigO(m\log m)$.  \\ \addlinespace[5pt]
            
            
            
            PD $(\downarrow)$& $\mathfrak{S}_m$ & $\frac{1}{N(y)}\sum_{j,\ell=1}^{m}1_{([y]_j<[y]_\ell)}1_{(\sigma(j)>\sigma(\ell))}$ & $\frac{m(m-1)}{2}$ &
            $\frac{m}{4}$ &
            $\text{MWFAS}(m)$. \\ \addlinespace[5pt]
            
            MAP $(\uparrow)$& $\mathfrak{S}_m$ & $ \frac{1}{|y|}\sum_{j=1}^m\frac{[y]_j}{\sigma(j)} 
                \sum_{\ell=1}^{\sigma(j)}y_{\sigma^{-1}(\ell)}$ & $\frac{m(m+1)}{2}$ & $\frac{1}{2}m\sqrt{\log (m+1)}$ & 
                $\text{QAP}(m)$.\\ \bottomrule
                
            \end{tabular}
    }
    \caption{Upper bounds for $\mathsf{A}$ in \cref{th:qs_discrete,th:lownoise,cor:improved-rates} and computational complexity of evaluating the QS estimator in \cref{eq:qs_estimator}, for a number of widely-used losses for multilabel/ranking problems. See \cref{sec:multilabel_and_ranking} for notation, \cref{sec:comp-considerations} for computational considerations and \cref{sec:constant-derivation} for the full derivation of the results.\label{fig:complexity_losses}\label{table:constants}}
\end{table*}

The key property of a loss being SELF is that, by linearity of the inner product,
\eqals{
\ell(z,x) &= \int L(z,y)dP(y|x) \\
&= \int\langle\phi(z),V\psi(y)\rangle_{\calH}dP(y|x) \\
&= \langle\phi(z),V g^*(x) \rangle_{\calH},
} 
with $g^*(x) =\int\psi(y)dP(y|x)$ being the conditional expectation of $\psi(y)$, given $x$. This means that in order to estimate $\ell$, we just need to find an estimator $\widehat{g}$ for the conditional expectation $g^*$, and then define $\widehat{\ell}(z,x)=\langle\phi(z),V \widehat{g}(x)\rangle_{\calH}$. To find a suitable estimator for $g^*$, note that $g^*$ can be written as the minimizer of the following quadratic surrogate (QS),
\eqal{\label{eq:least_squares}
    g^* = \argmin_{g:\calX\rightarrow\calH}\calR_\psi(g), 
}
where $\calR_\psi(g):= \int\|g(x)-\psi(y)\|_{\calH}^2 dP(x,y)$ is the {\em expected surrogate risk} of $g$. The quality of the surrogate estimator $g$ is measured in terms of the {\em surrogate excess risk}
$\calR_\psi(g) - \calR_\psi(g^*)$.
In particular, denote by $d:\calH\rightarrow\calZ$ the decoding function
$d(u) = \argmin_{z\in\calZ}\langle\phi(z), Vu\rangle_{\calH}$.
In \cite{ciliberto2016consistent} it is proven that  
by construction, the QS estimator is {\em Fisher consistent}, i.e., $\fstar=d\circ g^*$, with $\fstar, g^*$ as above. Moreover, for any $g:\calX \to \calH$, the {\em comparison inequality} holds
\begin{equation}\label{eq:comparison_inequality}
    \ERL(d\circ g) - \ERL(\fstar) \leq 2c_{V,\phi}\sqrt{\calR_{\psi}(g) - \calR_\psi(g^*)},
\end{equation}
where $c_{V,\phi}=\sup_{z\in\calZ}\|V^*\phi(z)\|_{\calH}$.
In the next paragraph we recall how to devise a suitable estimator of~$g^*$.

\textbf{ The QS estimator depends only on $L$.}~
Given a finite dataset $(x_i,y_i)_{i=1}^n$, an estimator $\widehat{g}$ for $g^*$ can be found by considering the characterization of $g^*$ in terms of \cref{eq:least_squares}. Indeed the problem in \cref{eq:least_squares} can 
be solved using kernel ridge regression (KRR) \cite{caponnetto2007optimal}. Let $k:\calX\times\calX\rightarrow\Rspace{}$ be a kernel on $\calX$ and $\calH_\calX$ the associated reproducing kernel Hilbert space (RKHS). Then given $\lambda > 0$, KRR reads
\begin{equation}\label{eq:krr}
    \widehat{g}_n \in \argmin_{g\in\calG}\frac{1}{n}\sum_{i=1}^n\|g(x_i)-\psi(y_i)\|_{\calH}^2
    + \lambda \|g\|_{\calG}^2,
\end{equation}
where $\calG$ is the space of Hilbert-Schmidt operators from $\calH_X$ to $\calH$, which is 
isometric to $\calH\otimes\calH_{\calX}$. The minimizer $\widehat{g}_n$ can be written in
closed form as $\widehat{g}_n(\cdot)=\sum_{i=1}^n\alpha_i(\cdot)\psi(y_i)\in\calG$ where $\alpha(x) = (\alpha_1(x),\dots,\alpha_n(x)) \in \R^n$ is defined by
\eqal{\label{eq:weights}
\alpha(x) = (K + n\lambda I)^{-1}K_x,
}
with $K_x = (k(x,x_1),\dots, k(x,x_n)) \in \R^n$ and $K \in \R^{n\times n}$ is defined by $K_{ij} = k(x_i, x_j)$.

The key property here, is that due to the fact that $\widehat{g}_n$ is linear in the $\psi(y_i)$'s,
then $\widehat{\ell}(z,x)$ does not explicitly depend on the surrogate space $\calH$, indeed
$\widehat{\ell}(z,x) = \left\langle\phi(z), V\left(\sum_{i=1}^n\alpha_i(x)\psi(y_i)\right)\right\rangle_{\calH} = \sum_{i=1}^n\alpha_i(x)L(z,y_i).$ 
Hence, the final estimator $\widehat{f}_n = d \circ \widehat{g}_n$ can be written as 
\begin{equation}\label{eq:qs_estimator}
    \widehat{f}_n(x)  = \argmin_{z\in\calZ}\sum_{i=1}^n\alpha_i(x)L(z,y_i).
\end{equation}
Finally, by combining the comparison inequality with results on the convergence of $\widehat{g}_n$ to $g^*$ (see e.g.~\cite{caponnetto2007optimal}), the following theorem holds.
\begin{theorem}[Thm.~5 of \cite{ciliberto2016consistent}]\label{thm:base-rates}
Let $n \in \N$, $\lambda_n=n^{-1/2}$ and $\tau > 0$. If $L$ is SELF and $g^*\in\calG$,
then the following holds with probability at least $1-8 e^{-\tau}$,
\begin{equation}\label{eq:generalization_bound}
    \ERL(\widehat{f}_n) - \ERL(f^*) \leq C~c_{V,\phi}\kappa \|g^*\|_{\calG} \tau^2 n^{-1/4}.
\end{equation}
where $\kappa^2 = \sup_{x} k(x,x)$ and $C$ a universal constant.
\end{theorem}

\textbf{ Positioning of our contribution.}~
From a theoretical viewpoint the result above holds for any loss on discrete and finite $\calZ, \calY$, and shows a learning rate that is $\calO(n^{-1/4})$. Moreover, from a practical viewpoint, to define and evaluate the QS estimator in \cref{eq:qs_estimator} is enough to know only the loss $L$ and a kernel $k$ for $\calX$ (no knowledge of $\calH, \psi, \phi$ is required). These considerations show that the QS framework could be a good candidate to systematically solve learning problems with discrete outputs.

However, note that constants of the bound depend on the specific SELF decomposition for $L$.
If we use the one in \cite{ciliberto2016consistent},
$\calH=\Rspace{\cardZ}, \phi(z)=e_z, \psi(y) = e_y, V=L$, then 
the constant $c_{V,\phi}$ equals the spectral norm of the loss matrix 
$\|L\|$, which is exponentially large even for highly structured loss 
functions such as Hamming. In that case $\|L\|=2^{m-1}$, where $m$ is the number of labels (and a similar behaviour could affect $\|g^*\|_{\calG}$). Then \cref{eq:generalization_bound} can be totally uninformative if the constants of the rate are exponentially large \cite{osokin2017structured}.

In the next section, we prove that by using a suitable SELF-decomposition it is possible to find a version of \cref{eq:generalization_bound}, that depends only polynomially on the number of labels $m$. In particular we find the explicit constants for a number of widely used loss functions for ranking and multilabel learning. Finally we provide a refined generalization bound adaptive to the noise-level of the learning problem.


\section{Main Results} \label{sec:analysis}\label{sec:main}
In this section we study a specific SELF-decomposition for discrete losses, providing a generalization bound in the form of \cref{eq:generalization_bound}, with explicit constants depending on the specific loss chosen (\cref{th:qs_discrete}). In \cref{th:multilabel_ranking} and \cref{table:constants} we quantify the constants for a number of widely used loss functions for multilabeling and ranking problems, showing that they are always polynomial with respect to the number of labels and in many cases optimal (\cref{rmk:optimality}). Finally in \cref{th:lownoise} we generalize \cref{eq:generalization_bound} (and so the learning rate in \cite{ciliberto2016consistent}), introducing a Tsybakov-like noise condition for the structured prediction problem.

\subsection{Affine decomposition}\label{sec:affine}
Motivated by the limitations given by the possible exponential magnitude of the constants in the generalization bound in \cref{eq:generalization_bound}, we consider another SELF-decomposition of the loss, based on the following \textit{affine decomposition} of the loss matrix,

\begin{equation}\label{eq:low_rank_decomp}
    L = FU^\top + c \mathbf{1}.
\end{equation}

where $F \in \Rspace{|\calZ|\times r},U \in \Rspace{|\calY|\times r}$, $c\in\Rspace{}$ is a scalar and $\mathbf{1} \in \R^{|\calZ|\times|\calY|}$ is the matrix of ones, i.e. $\mathbf{1}_{ij} = 1$ and $r \in \N$. The minimum $r$
for which there exists a decomposition as \cref{eq:low_rank_decomp} is called
the \textit{affine dimension} of the loss $L$ and is denoted $\text{affdim}(L)$.

Note that the``centered" loss $L-c$ is SELF with
\begin{equation}\label{eq:self_discrete}
    \calH=\Rspace{r},~\phi(z) = F_z,~\psi(y)=U_y,~V=I_{r\times r},
\end{equation}
where $F_z$ is the $z$-th row of $F$ and $U_y$ the $y$-th row of $U$.
Using the decomposition above, the following theorem gives a new version of the bound \cref{eq:generalization_bound} specialized to discrete losses. Before giving the result, note that when we use the SELF-decomposition above for a loss, the conditional expectation $g^*$ is characterized by $g^*: \calX \to \R^r$, $g^*(x) = (g^*_j(x))_{j=1}^r$, for $g^*_j:\calX\to\R$ defined as $g^*_j(x) = {U^j}{}^\top \Pi(x)$, with $U^j \in \R^{|\calY|}$ the $j$-th column of $U$ and $\Pi: \calX \to \R^{|\calY|}$, $\Pi(x)_y = P(y|x)$ the conditional probability of $y$ given $x$. In particular  $g^*_j(x) \leq \max_{k \in \calY} |U_{kj}|$, ($U_{kj}$ is the $k,j$-th element of $U$). Finally, $\calG$ is isometric to $\calH_{\calX}^{r}$, since $\calH = \R^r$.

\begin{theorem}[Statistical complexity]\label{th:qs_discrete}Let $n \in \N, \tau >0$ and $\lambda_n = n^{-1/2}$.
    Assume that the loss $L$ decomposes as \cref{eq:low_rank_decomp}.
    If $g^* \in \calG$, we have that with probability
    $1-8e^{-\tau}$,
    \begin{equation}\label{eq:discrete_gen_bound}
    \ERL(\widehat{f}_n) - \ERL(\fstar) ~~\leq ~\mathsf{A}Q~C\kappa\tau^2 ~ n^{-1/4},
    \end{equation}
    where $C, \kappa$ are as in \cref{thm:base-rates},
    \begin{equation}\label{eq:A}
        \mathsf{A} = \sqrt{r}\|F\|_{\infty}U_{\max},
    \end{equation}
    $Q=\max_{1\leq j\leq r}\|g_j^*/U_{\max}\|_{\mathcal{H}_{\mathcal{X}}}$,
    $U_{\max}=\max_{j,k}|U_{kj}|$.
\end{theorem}

\begin{proof}
First, note that the excess risk $\ERL(f)-\ERL(f^*)$ is the same for $L$ and
for $L-c$ for any $c \in \R$, moreover both the definition of $\fstar$ and $\widehat{f}$ are invariant when $L-c$ is used instead of $L$. So we bound $\ERL(\widehat{f})-\ERL(f^*)$ with \cref{eq:generalization_bound} applied to $L-c$, with $c$ as in \cref{eq:low_rank_decomp}.

Applying \cref{thm:base-rates} with the affine decomposition in \cref{eq:self_discrete} for $L-c$ and the definition of $c_{V,\phi}$ by \cite{ciliberto2016consistent}, we have that $c_{V,\phi}:=\sup_{z\in\calZ}\|F_z\|_2=\|F\|_{\infty}$ and
$\|g^*\|_{\calG}^2=\sum_{j=1}^r\|g_j^*\|_{\calH_\calX}^2\leq r\max_{1\leq j\leq r}\|g_j^*\|^2_{\calH_\calX}$. The final result is obtained by multiplying and dividing by~$U_{\max}$.
\end{proof}

The theorem above is essentially a version of \cref{thm:base-rates} where we use the affine decomposition in \cref{eq:low_rank_decomp} for the loss $L$, making explicit the dependence of the constants on structural properties of the loss, like the {\em affine dimension}.

In particular, we explicitly identify three distinct terms $\mathsf{A}$, $Q$ and $C\kappa \tau^{2} n^{-1/4}$. The third term is completely explicit and does not dependent on the loss nor on the data distribution. It expresses the dependence of the statistical error with respect to the number of examples $n$ and the high probability confidence~$\tau$ ($C$ is a universal constant and $\kappa$ the constant of the kernel).
The second term depends on the data distribution $P$ and measures in a sense the ``regularity" of the most difficult regression scalar function $g^*_j$ defining the surrogate conditional expectation for the given loss. 
Note that the $Q$ is renormalized by $U_{\max}$ so is invariant to the magnitude of the representation vector $\psi$. 

Finally $\mathsf{A}$ depends only on the chosen loss and measures the cost of using the QS method as surrogate approach. In the next subsection we give sharp bounds on the constant $\mathsf{A}$ for many discrete losses used in practice, together with the computational complexity required to evaluate the QS estimator. In particular, we prove that, contrary to what suggested by \cite{osokin2017structured}, $\mathsf{A}$ depends only polynomially on the number of labels, making the QS method a good systematic approach to deal with discrete losses.

\subsection{Sharp constants for multilabel and ranking losses}\label{sec:multilabel_and_ranking}

In this section (\cref{th:multilabel_ranking}, \cref{fig:complexity_losses}) we characterize explicitly the constants introduced in \cref{th:qs_discrete}, for a number of widely used losses for multi-labeling and ranking problems. In particular we show that they depend only polynomially on the number of labels (or equivalently $\textrm{polylog}(|\calY|,|\calZ|)$). Moreover in \cref{rmk:optimality} we show that the bounds obtained for many of the considered losses are sharp in a precise sense \cite{ramaswamy2016convex}.
Finally we characterize the computational complexity of evaluating the QS estimator in \cref{eq:qs_estimator} for such losses. 

In the following we denote by $m \in \N$ the number of classes of a multilabel/ranking problem, by $\calP_m$ the power-set of $[m] = \{1,\dots,m\}$ and by $\mathfrak{S}_m$ the set of permutations of $m$-elements. In particular note that in the multilabel problems both the output space $\calZ$ and the observation space $\calY$ are equal to $\calP_m$, while in ranking $\calZ = \mathfrak{S}_m$ and $\calY =\{1,\ldots,R\}^m = [R]^m$, the set of observed relevance scores for the $m$ documents where $R$ is the highest relevance \cite{ravikumar2011ndcg}. Finally we denote by $[v]_j$ the $j$-th element of a vector $v$ and we identify $\calP_m$ with $\{0,1\}^m$, moreover $\sigma(j)$ is the $j$-th element of the permutation $\sigma$, for $\sigma \in \mathfrak{S}_m$, $j \in [m]$.

\begin{theorem}\label{th:multilabel_ranking}
    The constant $\mathsf{A}$ and the computational complexity of the QS estimator for the multilabel losses:
    0-1, block 0-1, Hamming, Prec@k, F-score and ranking losses:
    NDCG-type, PD and MAP, appearing in
    \cref{fig:complexity_losses} hold.
\end{theorem}
\begin{proof}

We sketch here the analyses for the Hamming loss and the
NDCG-type ranking measures. The complete analysis for all the losses is in \cref{sec:constant-derivation}.

{\em Hamming.}~ Let $m \in \N$ be the number of labels. We represent
each output element as a binary vector ($\calZ=\calY=\{0,1\}^m$).
We re-write the Hamming loss as
\begin{equation}
     L(y',y) = \frac{1}{2m} - \frac{1}{2m}\sum_{j=1}^m s_{j}(y') s_{j}(y),
\end{equation}
where $s_j(y)=2[y]_j-1$.
Hence, this corresponds to an affine decomposition by setting
\begin{equation*}
    F_z = -\frac{1}{2m}(s_j(z))_{j=1}^{m},~ U_y = (s_j(y))_{j=1}^{m},~ c=\frac{1}{2m}.
\end{equation*}
We have that $r=m, \|F\|_{\infty} = \frac{1}{2\sqrt{m}}, ~ U_{\max} = 1$.
This implies that $A=\frac{1}{2}$.
Finally, inference corresponds to $\widehat{f}_j(x) = (\text{sign}\left(\widehat{g}_j(x)\right)+1)/2$
where $\widehat{g}_j(x) = \sum_{i=1}^ns_j(y_i)\alpha_i(x)$. This is done in 
$\calO(m)$. 

{\em NDCG-type. }\cite{valizadegan2009learning, ravikumar2011ndcg, wang2013theoretical} ~ Let $\calZ=\mathfrak{S}_m$ be the set of permutations of $m$ elements and $\calY=[R]^m$ the set of relevance scores for $m$ documents. Let
the \textit{gain} $G:\Rspace{}\rightarrow\Rspace{}$ be an increasing function and the \textit{discount} vector $D=(D_j)_{j=1}^m$ be a coordinate-wise decreasing vector. The NDCG-type losses are defined as the normalized discounted sum of the gain of the relevance scores ordered by the predicted permutation: \begin{equation}\label{eq:ndcgtype-0}
    L(\sigma, r) = 1 - \frac{1}{N(r)}\sum_{j=1}^m G([r]_j)D_{\sigma(j)},
\end{equation}
where $N(r)=\max_{\sigma\in\mathfrak{S}_m}\sum_{j=1}^{m}G([r]_j)D_{\sigma(j)}$ is a normalizer.

Note that looking at \cref{eq:ndcgtype-0} we can directly write that
 $r=m$ and
\begin{equation}
    F_\sigma = -(D_{\sigma(j)})_{j=1}^m, U_r = \left(\frac{G([r]_j)}{N(r)}\right)_{j=1}^m, c=1.
\end{equation}

It follows that, $\|F\|_{\infty} = \|D\|_2,
    U_{\max} =D_{\max}G_{\max}$, so 
$A =  \sqrt{m}G_{\max} D_{\max}(\sum_{j=1}^mD_j^2)^{1/2}$.  For \cref{table:constants}, assume $D_1=1$. 
If we define the vector $v\in\Rspace{m}$ as
\begin{equation}
    v_j = \sum_{i=1}^n\frac{G([r_i]_j)\alpha_i(x)}{N(r_i)},\hspace{0.5cm}
    1\leq j\leq m,
\end{equation}
then inference corresponds to $f^*(x) =  \operatorname{argsort}_{\sigma\in\mathfrak{S}_m}(v)$,
which can be done in $\calO(m\log m)$ operations.
\end{proof}
The key result in \cref{table:constants} is that the generalization properties and the computational complexity of the algorithm are both polynomial in the number of labels $m$ (or equivalently $\textrm{polylog}(|\calY|,|\calZ|)$)
for all considered losses except the 0-1, which does not provide any structural information of the observation/output spaces $\calY, \calZ$. This theoretically explains why in discrete structured prediction and in particular multi-labeling and ranking, learning is possible even if the size of the output space is exponentially large compared to the number labels and, potentially, to the number of examples. Moreover this result shows that the Quadratic Surrogate is a valid candidate for systematically addressing learning problems with discrete losses both from a statistical and from a computational viewpoint (in contrast with what conjectured in \cite{osokin2017structured}).
\begin{remark}[On the sharpness of the QS estimator] \label{rmk:optimality}
It is natural to ask to what extent the statistical rates provided by 
\cref{th:qs_discrete} can be considered representative of the statistical difficulty of solving the problem in \cref{eq:optimal_f}.
Of course, formally answering this question necessarily requires a study of the corresponding minimax rates under certain priors. In particular, one would be interested in studying the dependence of those rates both in the number of samples and the size of the output space $\calZ$.

Although far from answering this question, we can provide a weaker notion of optimality on the framework of surrogate-based methods. In particular, by using the results in \cite{ramaswamy2016convex}, we prove that cannot exist a consistent convex surrogate that maps the discrete problem in a vector valued problem of lower dimension than $r$ (the 
one used by the QS estimator through the affine-decomposition) for the following losses: 0-1, block 0-1, Hamming, Prec@k, NDCG, PD and MAP (see \cref{sec:constant-derivation}). 
\end{remark}

More in  detail, for the Hamming loss we obtain that the statistical complexity of the problem is independent of the number of labels. Intuitively this is explained by the fact that the QS estimator corresponds to estimating the $m$ marginals independently. Our result is to be compared with \cite{osokin2017structured}, where they obtain a constant in the order of $\calO(m^2)$. For Prec@k, we obtain $\mathsf{A}=\sqrt{\frac{m}{k}}$, which is coherent with the intuition that the problem becomes more challenging when $k$ is fixed and $m$ increases. 
For the F-score the computational bound of the resulting algorithm is essentially in \cite{waegeman2014bayes}.
For the NDCG-type losses, $G:\Rspace{}\rightarrow\Rspace{}$, the {\em gain} is an increasing function and $D=(D_j)_{j=1}^m \in [0,1]^m$, the \textit{discount}, is a coordinate-wise decreasing vector. For this family of losses $\mathsf{A}$ depends crucially on the discount factor $D_j$, tending to $\sqrt{m}$ (the constant of Prec@1) for fast decaying $D_j$ and to $m$ for low decaying ones. For PD and MAP, estimating the surrogate function is statistically tractable, but both inference algorithms are NP-Hard (Minimum  Weight Feedback Arcset problem (MWFAS) for PD and an instance of Quadratic Assignment Problem (QAP) for MAP), as was already noted in \cite{ramaswamy2013convex}.

\subsection{Improved rates under low-noise assumption}\label{sec:lownoise}
Intuitively, if there is small noise at the decision boundary between different labels, then it should be statistically easier to discriminate between them. To formalize this intuition, we define the margin $\gamma(x)$ as
\begin{equation}
        \gamma(x) = \min_{z'\neq f^*(x)} \ell(z', x)- \ell(f^*(x), x).
\end{equation}
The margin function $\gamma$ measures the minimum suboptimality gap in terms of the Bayes risk. If for a given $x$ the margin is small, then its cost at the optimum is very close to the cost at a suboptimal label. We will say that the \textit{$p$-noise condition} is satisfied if
\begin{equation}\label{eq:tsybakov}
P_{\calX}(\gamma(X)\leq\varepsilon)=o(\varepsilon^p),
\end{equation}
where $P_{\calX}$ is the marginal of $P$ over $\calX$, with $p \geq 0$.
The parameter $p$ characterizes how fast the noise vanishes at the boundary and corresponds to no assumption when $p = 0$. 

Note that \cref{eq:tsybakov} is a generalization of the Tsybakov condition for binary classification \cite{tsybakov2004optimal} and of the condition in \cite{mroueh2012multiclass} for multi-class classification, to general discrete losses. Indeed, for the binary 0-1 loss ($\calY=\{-1,1\}$),
$\gamma(x) = |\Expect[Y|x]|$, so we recover the classical Tsybakov condition.
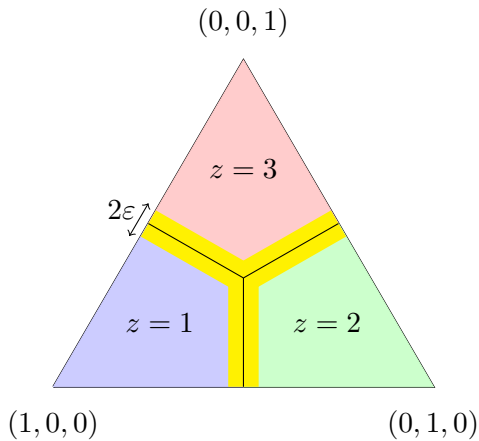
\begin{figure}[t!]
\begin{center}
    \begin{tikzpicture}[domain=0:1, xscale=5, yscale=5]
        \draw [thick, color=black] (0,0) -- (1, 0) -- (0.5, 0.866) -- (0, 0); 
        
        \draw [fill, color=blue!20] (0, 0) -- (0.25, 0.433) -- (0.5, 0.866/3)
        -- (0.5, 0) -- (0,0);
        \draw [fill, color=red!20] (0.25, 0.433) -- (0.5, 0.866) --
        (0.75, 0.433) -- (0.5, 0.866/3) -- (0.25, 0.433);
        \draw [fill, color=green!20] (0.5, 0) --(0.5, 0.866/3) --
        (0.75, 0.433) -- (1, 0) --  (0.5, 0);
        
        \draw[line width=4mm, color=yellow] (0.5, 0) -- (0.5, 0.866/3);
        \draw[line width=4mm, color=yellow] (0.25, 0.433) -- (0.5, 0.866/3);
        \draw[line width=4mm, color=yellow] (0.75, 0.433) -- (0.5, 0.866/3);
        
        \draw (0.5, 0) -- (0.5, 0.866/3);
        \draw(0.25, 0.433) -- (0.5, 0.866/3);
        \draw(0.75, 0.433) -- (0.5, 0.866/3);
        
        \draw[<->] (0.2, 0.4) -- +  (60:0.1);
        \node at (0.18, 0.47) {$2\varepsilon$};
        
        \node at (0, -0.1) {$(1,0,0)$};
        \node at (1, -0.1) {$(0,1,0)$};
        \node at (0.5, 0.966) {$(0,0,1)$};
        
        \node at (0.5, 2 * 0.866/3) {$z=3$};
        \node at (0.28, 0.17) {$z=1$};
        \node at (0.72, 0.17) {$z=2$};
    
    \end{tikzpicture}
    \caption{Generalized Tsybakov condition for discrete losses, \cref{eq:tsybakov}, in the case of multi-class. See \cref{ex:multiclass} for more details.}
    \label{fig:tsybakov}
\end{center}
\end{figure}

\begin{example}[Generalized Tsybakov for multiclass]\label{ex:multiclass}
    For every $P(y|x)$ in the simplex, one associates the corresponding
    optimal label as $z^*=\argmin_z\ell(z,x)$.
    \cref{fig:tsybakov} represents the partition of the simplex corresponding to the 0-1 loss for $\calZ=\calY=\{1,2,3\}$. In this case, $\gamma(x)$ corresponds to the distance to the boundary decision depicted in \cref{fig:tsybakov} and so $\{P(y|x)~|~\gamma(x)<\varepsilon\}$ corresponds to the yellow area. \cref{eq:tsybakov} says that the probability of falling in that region vanishes as $o(\varepsilon^p)$.
\end{example}

In the next theorem we improve the comparison inequality of \cref{eq:comparison_inequality} to take into account the generalized Tsybakov condition for discrete losses of \cref{eq:tsybakov}.

\begin{theorem}[Improved comparison inequality]\label{th:lownoise}
    Assume $\calY, \calZ$ to be finite and $\gamma$ to satisfy \cref{eq:tsybakov} for $p > 0$. Then the following holds
    
    \begin{enumerate}
    \item $1/\gamma\in L_p(P_{\calX})$.
    
    \item  Assume a decomposition as in \cref{eq:low_rank_decomp} for the loss $L$. Then, for any bounded measurable $g:\calX \to \calH$,
    \begin{equation}\label{eq:improved_comparison_inequality}
        \ERL(f) - \ERL(\fstar) \leq q \gamma_p^{\frac{1}{p+2}}\left(\calR_{\psi}(g) - \calR_\psi(g^*)\right)^{\frac{p+1}{p+2}},
    \end{equation}
    where $\gamma_p = \|1/\gamma\|_{L_p(P_{\calX})}$, $q = (16\|F\|_{\infty}^2)^{\frac{p+1}{p+2}}$.
    \end{enumerate}
\end{theorem}

The proof of the first part can be found in \cref{th:lp}, while the second part is \cref{th:improvedqscor}, both in the \cref{sec:lownoiseapp}.  As you can note, the comparison inequality of \cref{eq:comparison_inequality} is recovered when $p=0$ (i.e. when the generalized Tsybakov condition is always verified), while an exponent close to $1$, instead of $1/2$ is obtained when $p \gg 0$. Finally, by using the improved comparison inequality we refine the rates for the QS estimator in \cref{th:qs_discrete}.

\begin{corollary}[Improved rates]\label{cor:improved-rates}
    Under the $p$-noise condition, we have the following improvement on the generalization bound in \cref{eq:discrete_gen_bound},
    \begin{equation}
       \ERL(\hat{f}_n) - \ERL(\fstar) \leq 
       C\gamma_p^{\frac{1}{p+2}}
       \left(\mathsf{A}^2Q^2\kappa^2\tau^4n^{-\frac{1}{2}}\right)^{\frac{p+1}{p+2}},
    \end{equation}
    with $C$ universal constant and $\mathsf{A}, Q, \kappa$ as in \cref{th:qs_discrete}.
\end{corollary}

Note that the result of \cref{th:qs_discrete} is recovered for $p=0$ (always verified), while we obtain a learning rate essentially in the order of $n^{-1/2}$, instead of $n^{-1/4}$, in conditions of low-noise (i.e. $p \gg 0$). 


\section{Computational Considerations}\label{sec:comp-considerations}

\begin{table*}[ht!]
    \hspace{-0.15cm}
    \resizebox{!}{0.15\linewidth}{
        \begin{tabular}{@{}lllllllllll@{}}
            \toprule
            
            \textbf{ Multilabel}
            &  & bibtex & 
            birds & 
            CAL500 & corel5k & enron & mediamill & medical & scene & yeast \\ 
            \midrule
            & $n$ & 7395 & 
            645 & 
            502 & 5000 & 1702 & 43907 & 978 & 2407 & 2417 \\ 
             & $d$ & 1836 & 
            260 & 
            68 & 499  & 1001 & 120 & 1449 & 294 & 103 \\ 
            & $m$ & 159 & 
            19 & 
            174 & 374 & 53 & 101 & 45 & 6  & 14 \\ 
            \cmidrule{1-11}

            \multirow{3}{*}{0-1 ($\boldsymbol{\downarrow}$)}
            &
            THBM
            & 
            0.82
            & 
            0.57
            & 
            1.0
            & 
            0.99
            & 
            0.92
            & 
            0.93
            & 
            0.31
            & 
            0.49
            & 
            0.93
            \\ 
            &
            SSVM
            & 
            0.91
            & 
            0.53
            & 
            1.0
            & 
            0.99
            & 
            0.90
            & 
            1.0
            & 
            0.35
            & 
            0.51
            & 
            0.95
            \\
            &
            QS
            & 
            \textbf{0.78}
            & 
            \textbf{0.52}
            & 
            1.0
            & 
            \textbf{0.95}
            & 
            \textbf{0.86}
            & 
            \textbf{0.86}
            & 
            \textbf{0.29}
            & 
            \textbf{0.34}
            & 
            \textbf{0.76}
            \\
            \addlinespace
            
            \multirow{3}{*}{Ham ($\boldsymbol{\downarrow}$)}
            &
            THBM
            & 
            1.3e-2
            & 
            7.9e-2
            & 
            0.14
            & 
            1.1e-2
            & 
            \textbf{5.9e-2}
            & 
            \textbf{3.1e-2}
            & 
            \textbf{9.4e-3}
            & 
            0.11
            & 
            \textbf{0.26}
            \\ 
            &
            SSVM
            & 
            1.3e-2
            & 
            6.4e-2
            & 
            0.13
            & 
            1.0e-2
            & 
            7.1e-2
            & 
            8.7e-2
            & 
            1.07e-2
            & 
            0.11
            & 
            0.40
            \\
            &
            QS
            & 
            1.3e-2
            & 
            \textbf{4.9e-2}
            & 
            \textbf{0.14}
            & 
            \textbf{9.4e-3}
            & 
            8.6-2
            & 
            \textbf{3.1e-2}
            & 
            9.6e-3
            & 
            0.11
            & 
            0.42
            \\
            \addlinespace
            
            \multirow{3}{*}{F-score ($\boldsymbol{\uparrow}$)}
            &
            THBM
            & 
            0.44
            & 
            0.25
            & 
            0.46
            & 
            0.25
            & 
            0.51
            & 
            \textbf{0.56}
            & 
            0.80
            & 
            0.63
            & 
            \textbf{0.48}
            \\ 
            &
            SSVM
            & 
             0.19
            & 
            0.16
            & 
            0.33
            & 
            0.11
            & 
            0.49
            & 
            0.40
            & 
            0.74
            & 
            0.57
            & 
            \textbf{0.48}
             \\
            &
            QS
            & 
            \textbf{0.47}
            & 
            \textbf{0.28}
            & 
            \textbf{0.47}
            & 
            \textbf{0.26}
            & 
            \textbf{0.52}
            & 
            \textbf{0.56}
            & 
            \textbf{0.83}
            & 
            \textbf{0.68}
            & 
            0.47
            \\ \bottomrule
            \end{tabular}
    }~
    \resizebox{!}{0.15\linewidth}{
        \begin{tabular}{@{}lrl@{}}
            \toprule
            \textbf{ Ranking} 
            & \multicolumn{2}{r}{Ohsumed}  \\ \midrule
            &$n$  &  106 \\
            &$d$  & 25  \\
            &$m$  & 150  \\ \midrule \addlinespace \addlinespace
            \multirow{3}{*}{NDCG@3 ($\boldsymbol{\uparrow}$)} & SSVM & 0.47 \\ & QS & \textbf{0.51} \\ & & \\\addlinespace
            
            \multirow{3}{*}{NDCG@5 ($\boldsymbol{\uparrow}$)} & SSVM & 0.45 \\ & QS & \textbf{0.48} \\ & &\\\addlinespace 
            
            \multirow{3}{*}{NDCG@10 ($\boldsymbol{\uparrow}$)} & SSVM & 0.43 \\ & QS & \textbf{0.46} \\ & & \\
            \bottomrule
            \end{tabular}
    }
    \caption{Numerical results on real-world multilabeling and ranking datasets comparing our QS estimator, THBM \cite{zhang2014review} and SSVM \cite{joachims2006training}. $n$ is the size of the full dataset, $d$ the dimensionality of the data and $m$ the number of classes (multilabel), or the avg. number of query-document pairs (ranking). See \cref{sec:exp} for more details.\label{table:table_experiments}}
\end{table*}
As already observed in \cite{ciliberto2016consistent}: (1) the computation of the QS estimator (\cref{eq:qs_estimator}) is divided in {\em training step} and {\em inference step} (or {\em evaluation step}), (2) the SELF-decomposition of the loss is not needed to run the algorithm, but only to derive the theoretical guarantees. 
Here we show how the explicit knowledge of the affine decomposition of the loss can be useful to improve also the computational complexity of the method (its theoretical implications have been studied in \cref{sec:affine}). First we recall the training and test steps.

\textbf{ Training.}~
The training step requires only to have a kernel function $k$ over $\calX$ and to have access to the training input examples $(x_i)_{i=1}^n$. It consists essentially in computing the inverse of the kernel matrix necessary for the second step, i.e. $W = (K + \lambda n I)^{-1}$, with $K$ defined in \cref{eq:weights}. 

\textbf{ Evaluation.}~
The evaluation step requires only the knowledge of the loss $L$ and to have access to the train observations $(y_i)_{i=1}^n$. Given a test input point $x \in \calX$, it consists in: first, computing the coefficients $(\alpha_i(x))_{i=1}^n$ according to \cref{eq:weights}, i.e. $\alpha(x) = W K_x$, with the notation in \cref{eq:weights}; second predicting the output $z \in \calZ$ associated to the test input $x$, by solving \cref{eq:qs_estimator}.

\subsection{Using the affine decomposition to speed up the QS estimator}\label{affine-for-optimization}
Note that, to run the algorithm described above, only the loss $L$ and kernel $k$ are needed. This makes the QS-method (1) systematically applicable to any supervised learning problem with discrete loss, since it does not require to devise a specific surrogate for each loss (2) theoretically grounded with basic guarantees from \cite{ciliberto2016consistent} in terms of consistency and learning rates. Indeed note that the SELF-decomposition in terms of $\calH, \phi, \psi$ for the loss and in particular the affine decomposition of \cref{eq:low_rank_decomp} is needed only to prove the sharper guarantees in \cref{th:qs_discrete,th:multilabel_ranking,th:lownoise,cor:improved-rates}. 

However it is possible to additionally exploit the affine decomposition to have even a computational benefit for the presented algorithm, as we are going to show in the rest of the section.

\textbf{ Improved training when $U$ is known.}~
When we know the affine decomposition of the loss, we have $\calH=\R^r$ and $\psi(y) = U_y$, so we can compute explicitly the solution to \cref{eq:krr} \cite{caponnetto2007optimal}, $\widehat{g}_n:\calX \to \R^r$ that is $\widehat{g}_n(x) = \sum_{i=1}^n k(x,x_i) C_i$, where $C_i \in \R^r$ is the $i$-th row of $C \in \R^{n \times r}$, the solution of the linear system 
$$
(K+\lambda n I) C = \Psi^\top,
$$
with $\Psi=(\psi(y_1),\dots,\psi(y_n)) \in \R^{r \times n}$. This is the same as solving $r$ scalar KRR problems independently and its computation can be efficiently reduced from essentially $\calO(n^3r)$ to $\calO(n \sqrt{n} r)$ via suitable random projection techniques \cite{Smola:2000:SGM:645529.657980,rahimi2008random,rudi2017falkon}.

\textbf{Improved evaluation when $F$ is known.}~
Given a test point $x\in\calX$, first we evaluate $\theta:= \widehat{g}_n(x) \in \R^r$, requiring essentially $\calO(nr)$  (up to $\calO(\sqrt{n}r)$ by using random projection techniques \cite{Smola:2000:SGM:645529.657980,rahimi2008random,rudi2017falkon}). Then we use the characterization of $\widehat{f}_n(x) = (d \circ \widehat{g}_n)(x) = d \circ \theta$, to obtain the equivalent problem
\begin{equation}\label{eq:inference}
\min_{z \in \calZ}~F_z \cdot \theta,
\end{equation}
where $F_z \in \R^r$ is the $z$-th row of $F$ (see \cref{eq:low_rank_decomp,eq:self_discrete}) and $(\cdot)$ the dot-product. The computational complexity of \cref{eq:inference} (we denote it by $\text{INF}_{F}(\cardZ)$) has been devised for a number of widely used losses in \cref{th:multilabel_ranking}, \cref{table:constants} (see \cref{sec:constant-derivation} for the proofs).

\section{Numerical Experiments}\label{sec:exp}
We perform numerical experiments for the QS-estimator on multilabeling (9 datasets \cite{tsoumakas2011mulan}) and ranking problems (1 dataset \cite{hersh1994ohsumed}), see \cref{table:table_experiments}. We use three evaluation measures for multilabel, namely, 0-1, Hamming and F-score, and NDCG@k for ranking (in the NDCG-type family \cite{valizadegan2009learning}), which have been theoretically analysed in \cref{sec:analysis} and \cref{table:constants}. All experiments are performed using 60\% of the dataset for training, 20\% for validation and 20\% for testing. We compare the performance of the QS-estimator with a threshold-based method, which we denote by THBM \cite{zhang2014review}, and the Structural SVM \cite{joachims2006training} (SSVM). THBM is a common method for multilabelling where learning is done in two stages. The method first estimates the $m$ marginals $\widehat{g}_j(\cdot)$ and then learns the best threshold function $\widehat{t}(\cdot)$ minimizing via least-squares the measure of interest. The inference is performed via thresholding the estimated marginals by $\widehat{t}(x)$ (see Sec. 2 in \cite{zhang2014review}). The SSVM corresponds to the multilabel-SVM \cite{finley2008training}, which is an instance of the SSVM with unary potentials that optimizes the Hamming loss. Note that we have used the same multilabel-SVM for all multilabel losses; for the F-score, there is no principled way of optimizing the measure with SSVMs. The experimental results in \cref{table:table_experiments} show that the QS-estimator outperforms the other methods for 0-1 loss and F-score. Indeed, the method depends on the loss and is designed to be consistent with it. THBM achieves approximatively the same accuracy for Hamming as it is based on estimating the marginals, while the SSVM is proven to be inconsistent even in this case \cite{gao2011consistency}, as the experimental result empirically shows. For the ranking experiment, we have used the SSVM from \cite{joachims2006training} called RankSVM as baseline to compare with the QS-estimator. The algorithm corresponding to the QS-estimator for NDCG, which corresponds to the one in \cite{ravikumar2011ndcg} for this measure, outperforms the SSVM. This highlights the importance of consistency in learning, and the importance of making the algorithm dependent on the measure willing to use for evaluation.


\section{Related Works \& Discussion} \label{sec:related_work}
While the QS for structured prediction generalizes the QS for binary classification, Structural SVMs (SSVMs) \cite{tsochantaridis2004support, crammer2001algorithmic} and Conditional Random Fields (CRFs) \cite{Lafferty:2001:CRF:645530.655813,settles2004biomedical,sutton2012introduction} generalize the binary SVM and logistic regression to the structured case. All of them are surrogate methods based on minimizing the expected risk of a certain surrogate loss $S(v, y):\calC\times\calY\rightarrow{}\Rspace{}$ in a convex surrogate space $\calC$.
The corresponding surrogates are $S_{\text{QS}}(v, y)=\|v-U_y\|_{\Rspace{r}}^2$, 
$S_{\text{SSVM}}(v, y)=\max_{y'\in\calY}(v_{y'}+L(y',y))-v_y$ and
$S_{\text{CRF}}(v, y)=\log(\sum_{y'\in\calY}\exp v_{y'})-v_y$ (See Examples in \cref{sec:prerequisites}) for QS, SSVM and CRF, respectively. SSVMs and CRFs exploit the structure of the problem by decomposing each output element into cliques and considering only the features on this parts. This is necessary for the tractability of the methods. Moreover, for SSVMs, the loss $L$ must decompose into these cliques to make possible the maximization inside the surrogate, often called augmented inference. The clique decomposability of the loss, can be seen as a low rank decomposition, analogous to our SELF-decomposition.
While the QS has attractive statistical properties, it is generally not the case for the other surrogate methods. CRFs are only consistent for the 0-1 loss in the case that the model is well-specified \cite{sutton2012introduction}. This lack of calibration to a given loss is an important drawback of this method \cite{volkovs2011loss}. SSVMs are in general not Fisher consistent, even for the 0-1 loss, for which is only consistent if the problem is deterministic, i.e, there always exists a majority label $y$ with probability larger than 1/2 \cite{zhang2004statistical}. 

\textbf{ QS for structured prediction.}~ 
\cite{ramaswamy2013convex}
proposed the QS through an affine decomposition of the loss and derived Fisher consistency of the corresponding surrogate method. They analyzed the inference algorithms for Prec@k, ERU (NDCG-type measure that we study in \cref{sec:constant-derivation}), PD and MAP. As Fisher consistency is a property only at the optimum, their analysis is not able to provide any statistical guarantees. \cite{ravikumar2011ndcg} analyses consistency and calibration properties for the QS specialized for NDCG-type losses. In particular, they highlight the fact that estimating the normalized relevance scores is key to be consistent, which is a property that follows directly from our framework.

As far as we know, \cite{osokin2017structured} is the only work that addresses the learning complexity of general discrete losses for structured prediction. They consider a different QS surrogate than ours, which could be potentially intractable to compute since it is defined on the space of labels (even when the loss is low-rank)  $\Expect \|Fg(X) - L(\cdot,Y)\|_{\Rspace{\cardZ}}^2$, and not in the low dimensional space of the decomposition $\Expect\|g(X)-U_Y\|_{\Rspace{r}}^2$.
They also obtain rates of the form $\propto \mathsf{A} n^{-1/4}$, however, their constants are always larger than ours and computed explicitly only for a small number of loss functions. In particular, for $\mathsf{A}$, they obtain
$\calO(2^m),\calO(b),\calO(m^2)$, while we obtain $\calO(2^{m/2}),\calO(\sqrt{b}),
\calO(1)$ for the 0-1, block 0-1 and Hamming, respectively. In addition, our constants are interpretable and most of them can be proven to be optimal (in the sense explained in \cref{rmk:optimality}). Finally we provide a refined bound adaptive to the noise of the problem as in \cref{cor:improved-rates}.

To conclude, \cite{ramaswamy2016convex} introduces and studies the concept of convex calibration dimension. We use their lower bound on this quantity to study the optimality of the dimension of the QS as reported in \cref{rmk:optimality}.

\subsection*{Acknowledgements}
This work was supported by the European Research Council (project Sequoia 724063).


\bibliography{biblio}

\begin{thebibliography}{10}

\bibitem{bakir2007predicting}
G{\"o}khan BakIr, Thomas Hofmann, Bernhard Sch{\"o}lkopf, Alexander~J. Smola,
  Ben Taskar, and S.V.N. Vishwanathan.
\newblock {\em Predicting Structured Data}.
\newblock MIT press, 2007.

\bibitem{crammer2001algorithmic}
Koby Crammer and Yoram Singer.
\newblock On the algorithmic implementation of multiclass kernel-based vector
  machines.
\newblock {\em Journal of Machine Learning Research}, 2(Dec):265--292, 2001.

\bibitem{read2011classifier}
Jesse Read, Bernhard Pfahringer, Geoff Holmes, and Eibe Frank.
\newblock Classifier chains for multi-label classification.
\newblock {\em Machine Learning}, 85(3):333, 2011.

\bibitem{pedregosa2017consistency}
Fabian Pedregosa, Francis Bach, and Alexandre Gramfort.
\newblock On the consistency of ordinal regression methods.
\newblock {\em The Journal of Machine Learning Research}, 18(1):1769--1803,
  2017.

\bibitem{Lafferty:2001:CRF:645530.655813}
John~D. Lafferty, Andrew McCallum, and Fernando C.~N. Pereira.
\newblock Conditional random fields: Probabilistic models for segmenting and
  labeling sequence data.
\newblock In {\em Proceedings of the Eighteenth International Conference on
  Machine Learning}, ICML '01, pages 282--289, San Francisco, CA, USA, 2001.
  Morgan Kaufmann Publishers Inc.

\bibitem{settles2004biomedical}
Burr Settles.
\newblock Biomedical named entity recognition using conditional random fields
  and rich feature sets.
\newblock In {\em Proceedings of the International Joint Workshop on Natural
  Language Processing in Biomedicine and its Applications}, pages 104--107.
  Association for Computational Linguistics, 2004.

\bibitem{volkovs2011loss}
Maksims~N. Volkovs, Hugo Larochelle, and Richard~S. Zemel.
\newblock Loss-sensitive training of probabilistic conditional random fields.
\newblock {\em arXiv preprint arXiv:1107.1805}, 2011.

\bibitem{tsochantaridis2004support}
Ioannis Tsochantaridis, Thomas Hofmann, Thorsten Joachims, and Yasemin Altun.
\newblock Support vector machine learning for interdependent and structured
  output spaces.
\newblock In {\em Proceedings of the Twenty-first International Conference on
  Machine Learning}, page 104. ACM, 2004.

\bibitem{joachims2006training}
Thorsten Joachims.
\newblock Training linear svms in linear time.
\newblock In {\em Proceedings of the 12th ACM SIGKDD International Conference
  on Knowledge Discovery and Data Mining}, pages 217--226. ACM, 2006.

\bibitem{tewari2007consistency}
Ambuj Tewari and Peter~L. Bartlett.
\newblock On the consistency of multiclass classification methods.
\newblock {\em Journal of Machine Learning Research}, 8(May):1007--1025, 2007.

\bibitem{ciliberto2016consistent}
Carlo Ciliberto, Lorenzo Rosasco, and Alessandro Rudi.
\newblock A consistent regularization approach for structured prediction.
\newblock In {\em Advances in Neural Information Processing Systems}, pages
  4412--4420, 2016.

\bibitem{osokin2017structured}
Anton Osokin, Francis Bach, and Simon Lacoste-Julien.
\newblock On structured prediction theory with calibrated convex surrogate
  losses.
\newblock In {\em Advances in Neural Information Processing Systems}, pages
  302--313, 2017.

\bibitem{tsybakov2004optimal}
Alexander~B Tsybakov.
\newblock Optimal aggregation of classifiers in statistical learning.
\newblock {\em The Annals of Statistics}, 32(1):135--166, 2004.

\bibitem{steinwart2008support}
Ingo Steinwart and Andreas Christmann.
\newblock {\em Support Vector Machines}.
\newblock Springer Science \& Business Media, 2008.

\bibitem{caponnetto2007optimal}
Andrea Caponnetto and Ernesto De~Vito.
\newblock Optimal rates for the regularized least-squares algorithm.
\newblock {\em Foundations of Computational Mathematics}, 7(3):331--368, 2007.

\bibitem{ramaswamy2016convex}
Harish~G. Ramaswamy and Shivani Agarwal.
\newblock Convex calibration dimension for multiclass loss matrices.
\newblock {\em The Journal of Machine Learning Research}, 17(1):397--441, 2016.

\bibitem{ravikumar2011ndcg}
Pradeep Ravikumar, Ambuj Tewari, and Eunho Yang.
\newblock On ndcg consistency of listwise ranking methods.
\newblock In {\em Proceedings of the Fourteenth International Conference on
  Artificial Intelligence and Statistics}, pages 618--626, 2011.

\bibitem{valizadegan2009learning}
Hamed Valizadegan, Rong Jin, Ruofei Zhang, and Jianchang Mao.
\newblock Learning to rank by optimizing ndcg measure.
\newblock In {\em Advances in Neural Information Processing Systems}, pages
  1883--1891, 2009.

\bibitem{wang2013theoretical}
Yining Wang, Liwei Wang, Yuanzhi Li, Di~He, and Tie-Yan Liu.
\newblock A theoretical analysis of ndcg type ranking measures.
\newblock In {\em Conference on Learning Theory}, pages 25--54, 2013.

\bibitem{waegeman2014bayes}
Willem Waegeman, Krzysztof Dembczy{\'n}ki, Arkadiusz Jachnik, Weiwei Cheng, and
  Eyke H{\"u}llermeier.
\newblock On the bayes-optimality of f-measure maximizers.
\newblock {\em The Journal of Machine Learning Research}, 15(1):3333--3388,
  2014.

\bibitem{ramaswamy2013convex}
Harish~G. Ramaswamy, Shivani Agarwal, and Ambuj Tewari.
\newblock Convex calibrated surrogates for low-rank loss matrices with
  applications to subset ranking losses.
\newblock In {\em Advances in Neural Information Processing Systems}, pages
  1475--1483, 2013.

\bibitem{mroueh2012multiclass}
Youssef Mroueh, Tomaso Poggio, Lorenzo Rosasco, and Jean-Jeacques Slotine.
\newblock Multiclass learning with simplex coding.
\newblock In {\em Advances in Neural Information Processing Systems}, pages
  2789--2797, 2012.

\bibitem{zhang2014review}
Min-Ling Zhang and Zhi-Hua Zhou.
\newblock A review on multi-label learning algorithms.
\newblock {\em IEEE Transactions on Knowledge and Data Engineering},
  26(8):1819--1837, 2014.

\bibitem{Smola:2000:SGM:645529.657980}
Alex~J. Smola and Bernhard Sch\"{o}kopf.
\newblock Sparse greedy matrix approximation for machine learning.
\newblock In {\em Proceedings of the Seventeenth International Conference on
  Machine Learning}, ICML '00, pages 911--918, San Francisco, CA, USA, 2000.
  Morgan Kaufmann Publishers Inc.

\bibitem{rahimi2008random}
Ali Rahimi and Benjamin Recht.
\newblock Random features for large-scale kernel machines.
\newblock In {\em Advances in Neural Information Processing Systems}, pages
  1177--1184, 2008.

\bibitem{rudi2017falkon}
Alessandro Rudi, Luigi Carratino, and Lorenzo Rosasco.
\newblock Falkon: An optimal large scale kernel method.
\newblock In {\em Advances in Neural Information Processing Systems}, pages
  3891--3901, 2017.

\bibitem{tsoumakas2011mulan}
Grigorios Tsoumakas, Eleftherios Spyromitros-Xioufis, Jozef Vilcek, and Ioannis
  Vlahavas.
\newblock Mulan: A java library for multi-label learning.
\newblock {\em Journal of Machine Learning Research}, 12(Jul):2411--2414, 2011.

\bibitem{hersh1994ohsumed}
William Hersh, Chris Buckley, T.J. Leone, and David Hickam.
\newblock Ohsumed: an interactive retrieval evaluation and new large test
  collection for research.
\newblock In {\em SIGIR’94}, pages 192--201. Springer, 1994.

\bibitem{finley2008training}
Thomas Finley and Thorsten Joachims.
\newblock Training structural svms when exact inference is intractable.
\newblock In {\em Proceedings of the 25th International Conference on Machine
  Learning}, pages 304--311. ACM, 2008.

\bibitem{gao2011consistency}
Wei Gao and Zhi-Hua Zhou.
\newblock On the consistency of multi-label learning.
\newblock In {\em Proceedings of the 24th Annual Conference on Learning
  Theory}, pages 341--358, 2011.

\bibitem{sutton2012introduction}
Charles Sutton, Andrew McCallum, et~al.
\newblock An introduction to conditional random fields.
\newblock {\em Foundations and Trends{\textregistered} in Machine Learning},
  4(4):267--373, 2012.

\bibitem{zhang2004statistical}
Tong Zhang.
\newblock Statistical analysis of some multi-category large margin
  classification methods.
\newblock {\em Journal of Machine Learning Research}, 5(Oct):1225--1251, 2004.

\bibitem{bartlett2006convexity}
Peter~L. Bartlett, Michael~I. Jordan, and Jon~D. McAuliffe.
\newblock Convexity, classification, and risk bounds.
\newblock {\em Journal of the American Statistical Association},
  101(473):138--156, 2006.

\bibitem{calauzenes2012non}
Cl{\'e}ment Calauzenes, Nicolas Usunier, and Patrick Gallinari.
\newblock On the (non-) existence of convex, calibrated surrogate losses for
  ranking.
\newblock In {\em Advances in Neural Information Processing Systems}, pages
  197--205, 2012.

\end{thebibliography}
\newpage

\appendix \label{sec:appendix}


\author{ \Large{\textbf{Supplementary Material}}}
\date{}
\maketitle
\begin{itemize}
    \item [\textbf{A.}] \textit{\textbf{Calibration and fast rates for surrogate methods} }
    \begin{itemize}
        \item [\textbf{A.1. }] \textit{\textbf{Prerequisites on surrogate methods}}
         \item [\textbf{A.2. }] \textit{\textbf{Calibration}}
         \item [\textbf{A.3. }] \textit{\textbf{Improved calibration under low noise}}
    \end{itemize}
    \item [\textbf{B.}] \textit{\textbf{Multilabel and ranking losses}}
    \begin{itemize}
        \item [\textbf{B.1. }] \textit{\textbf{Prerequisites}}
         \item [\textbf{B.2. }] \textit{\textbf{On the optimality of the QS}}
         \item [\textbf{B.3. }] \textit{\textbf{Analysis of the losses}}
    \end{itemize}
\end{itemize}
\section{Calibration and fast rates for surrogate methods} \label{sec:lownoiseapp}
The goal of \cref{sec:lownoiseapp} is to provide a generic method to sistematically improve the relation between excess risks of surrogate methods. Our analysis is a generalization of the one in \cite{bartlett2006convexity}, which was done for binary classification under 0-1 loss, to the case of general discrete losses.

In \cref{sec:prerequisites}, we introduce the basic quantities used for the analysis of surrogate methods. Then, in \cref{sec:calibration} we focus on the central concept of {\em calibration}, which is key to study the statistical properties of these methods. In particular, we will re-derive the calibration properties of the Quadratic Surrogate (QS), which were proved in \cite{ciliberto2016consistent}. Finally, in \cref{sec:lownoiseappsec}, we derive our main result, which generalizes the Tsybakov condition, existing for multiclass and
binary \cite{mroueh2012multiclass, tsybakov2004optimal} classification.


\subsection{Prerequisites on surrogate methods}\label{sec:prerequisites}

Given a loss $L:\calZ\times\calY\rightarrow\Rspace{}$ and
a probability measure $P$ on $\calX\times\calY$, recall that the goal of supervised learning is to find the function $f^*$ that minimizes the {\em expected risk} $\calE(f)$ of the loss,
\begin{equation}
    f^*(x) = \argmin_{z\in\calZ}\ell(z,x),\quad
    \ERL(f) = \Expect\ell(f(X), X),
\end{equation}
where $\ell(z,x)=\int L(z,Y)dP(Y|x)$ is the {\em Bayes risk}. The goal of surrogate methods is to design a tractable {\em surrogate loss} $S:\calC\times\calY\rightarrow\Rspace{}$ defined on a {\em surrogate space} $\calC$, such that when approximately minimized by a
{\em surrogate function} $\widehat{g}:\calX\rightarrow\calC$, then it produces a good estimator $\widehat{f}$ of $f^*$.
The mapping from $\widehat{g}$ to $\widehat{f}$ is performed with a {\em decoding function} $d:\calC\rightarrow\calZ$.

For a given surrogate $S$, we define the following quantities,
\begin{equation}
    g^*(x) = \argmin_{v\in\calC}W(v,x), \quad W(v,x)=\int S(v, Y)dP(Y|x)\quad
    \calR(g) = \Expect W(g(X), X),
\end{equation}
were here, $g^*$ is the \textit{optimal surrogate function}, 
$W(v,x)$ is the \textit{Bayes surrogate risk} and $\calR(g)$ is the \textit{expected surrogate risk} of $g$.

An important requirement for a surrogate method is the so-called {\em Fisher consistency}, which says that the optimum $g^*$ of the surrogate $S$ gives the optimum $f^*$ of the loss $L$. It can be written as $f^*=d\circ g^*$.
\begin{example}[Surrogate elements for the QS]
    In the case of the QS, we have that ,
    \begin{equation}
        S(v,y)=\|v-U_y\|_{\Rspace{r}}^2, \quad\calC=\Rspace{r},\quad d(v)=\argmin_{z\in\calZ}F_z\cdot v,
    \end{equation}
    and its Bayes excess risk $W(\widehat{g}(x),x)-W(g^*(x),x)$ has the following form,
    \begin{equation}\label{eq:excessbayesqs}
         W(\widehat{g}(x),x)-W(g^*(x),x) = \|\widehat{g}(x) - g^*(x)\|_2^2.
    \end{equation}
    Moreover, it is Fisher consistent by construction (\cite{ciliberto2016consistent}).
\end{example}
\begin{example}[Surrogate elements for the CRFs and SSVMs] (Assume $\calZ=\calY$) Conditional Random Fields (CRFs) and Structural SVMs (SSVMs) are also surrogate methods for structured prediction. In this case, 
they split the output into a set of parts/cliques $C$ as $\{\calY_c\}_{c\in C}$, which encode the structure of the output set. Then,
they both consider
\begin{equation}
    \calC=\Rspace{r}, \quad d(v) = \argmax_{z\in\calZ} \sum_{c\in C}v_{z_c},
\end{equation}
where $r=\sum_{c\in C}|\calY_c|$.
The surrogate for CRFs has the following form (note that it does not dependent on any $L$),
\begin{equation}
    S(v, y) = \log\left(\sum_{y'\in\calY}\exp\left(\sum_{c\in C}v_{y'_c}\right)\right) - 
    \sum_{c\in C}v_{y_c}.
\end{equation}
For SSVM, one assumes that the loss decomposes accordingly to the structure given by $C$. Then, it takes the following form,
\begin{equation}
     S(v, y) = \max_{y'\in\calY} \left\{\sum_{c\in C}
    \left(\{L(y_c, y'_c) + v_{y'_c}\}\right)\right\} - \sum_{c\in C}v_{y_c}.
\end{equation}
\end{example}


\subsection{Calibration}\label{sec:calibration}
Fisher consistency is an essential property of a surrogate method, nevertheless, it is only a property at the optimum. In practice the surrogate will be never optimized exactly, this is why it is important to study the concept of \textit{calibration}, i.e, how the excess risk of the surrogate relates to the excess risk of the loss of interest.

This concept is formalized through the following \cref{defn:calibration}.
\begin{defn}[Calibration and Calibration function]\label{defn:calibration}
We say that a surrogate $S$ is {\em calibrated} w.r.t a loss $L$ if there exists a
convex function $H_{L,S}:\Rspace{}\rightarrow\Rspace{}$ with $H_{L,S}(0)=0$ and  positive in $(0,\infty)$, such that,
\begin{equation}\label{eq:calibrationfunctionbayes}
    H_{L,S}(\ell(d\circ g(x), x) - \ell(f^*(x), x)) \leq W(g(x), x) - W(g^*(x),x),
\end{equation}
for every $x\in\calX$.
\end{defn}
Calibration means that for every $x$, one can control the excess of the Bayes risk by the excess Bayes risk of the surrogate. 

Let's re-derive the form of the calibration function for the QS.
\begin{lemma}[Calibration function for QS \cite{ciliberto2016consistent}]\label{lem:calibrationqs}
Assumption 1 holds for the QS with
\begin{equation}\label{eq:calfunction}
    H_{L,S}(\varepsilon)=\frac{\varepsilon^2}{4\|F\|_{\infty}^2}
\end{equation}
\end{lemma}
\begin{proof}
Let's first decompose the Bayes risk into two terms $A$ and $B$:
\begin{equation}
    \ell(\widehat{f}(x), x) - \ell(f^*(x), x) = \{\ell(\widehat{f}(x), x) - \widehat{\ell}(\widehat{f}(x),x)\} + \{\widehat{\ell}(\widehat{f}(x),x) - \ell(f^*(x), x)\} = A + B.
\end{equation}
The first term, clearly $A\leq \sup_{z\in\calZ}|\widehat{\ell}(z, x) - \ell(z,x)|$.
For the second term , we use the fact that for any given two functions $\eta,\zeta:\calZ\rightarrow\Rspace{}$, it holds that $|\min_z\eta(z) - \min_z\zeta(z)|\leq \sup_z|\eta(z)-\zeta(z)|$. As $\widehat{f}(x)$ minimizes $\widehat{\ell}(\cdot,x)$ and $f^*(x)$
minimizes $\ell(\cdot,x)$, we can conclude also that $B\leq \sup_{z\in\calZ}|\widehat{\ell}(z, x) - \ell(z,x)|$. Using the fact that $\widehat{\ell}(z,x)=F_z\widehat{g}(x)$ and $\ell(z,x) = F_zg^*(x)$, we can conclude that,
\begin{equation}
    (\ell(f(x), x) - \ell(f^*(x), x))^2 \leq 2\sup_{z\in\calZ}(\widehat{\ell}(z, x) - \ell(z,x))^2
    = 4\|F\|_{\infty}^2\|\widehat{g}(x) - g^*(x)\|_2^2.
\end{equation}
Re-arranging and using \cref{eq:excessbayesqs} gives the final result.
\end{proof}

The following important Theorem shows how \cref{eq:calfunction} translates into a relation between excess risks, which are the quantities that we are ultimately interested at.
\begin{theorem}[From Bayes risks to risks]\label{th:tofullrisks} Suppose Assumption 1 holds. Then,
\begin{equation}\label{eq:relationrisks}
    H_{L,S}(\ERL(f) - \ERL(f^*)) \leq\calR(g) -\calR(g^*)
\end{equation}
\end{theorem}
\begin{proof}
This is a simple application of Jensen inequality.
\begin{align*}
    &H_{L,S}(\ERL(f) - \ERL(f^*)) = H_{L,S}(\Expect_X(\ell(d\circ g(x), x) - \ell(f^*(x), x)))\\
    &\leq \Expect_X H_{L,S}(\ell(d\circ g(x), x) - \ell(f^*(x), x)) = \Expect_X W(g(x), x) - W(g^*(x),x) = \calR(g) -\calR(g^*)
\end{align*}
\end{proof}
If we combine \cref{th:tofullrisks} with \cref{lem:calibrationqs}, we obtain the comparison inequality for the QS.
\begin{corollary}[Comparison inequality for QS \cite{ciliberto2016consistent}]
    For the QS, we have that
    \begin{equation}
        \ERL(f) - \ERL(f^*) \leq 2\|F\|_{\infty}\sqrt{\calR(g) -\calR(g^*)}
    \end{equation}
\end{corollary}


\subsection{Improved calibration under low noise}\label{sec:lownoiseappsec}
\cref{th:tofullrisks} gives the ability to translate learning rates of the surrogate to learning rates of the full risk. However, as we will show, \cref{eq:relationrisks} can be loose in the presence of low noise at the boundary decision. 

To formalize this, we will improve the result from the relation given by \cref{th:tofullrisks} under the $p$-noise assumption. We recall that the $p$-noise condition states that
\begin{equation}
    P_{\calX}(\gamma(X)<\varepsilon)=o(\varepsilon^p),
\end{equation}
where $\gamma(x) = \min_{z'\neq f^*(x)} \ell(z', x)- \ell(f^*(x), x)$, is called the margin, and is defined as the minimum suboptimality gap between labels.

We have the following \cref{th:lp}.
\begin{lemma}\label{th:lp}
If the p-noise condition holds, then $1/\gamma\in L_p(P_{\calX})$.
\end{lemma}
\begin{proof}
\begin{equation}
     \|1/\gamma\|_{L_p(P_{\calX})}^p = \Expect1/\gamma(X)^{p} = \int_{0}^\infty  pt^{p-1}P_{\calX}(1/\gamma(X)>t)dt 
    = \int_{0}^\infty pt^{p-1} P_{\calX}(\gamma(X)<t^{-1})dt.
\end{equation}
The integral converges if $P_{\calX}(\gamma(X)<t^{-1})$ decreases faster than $t^{-p}$.
\end{proof}

Let's now define the error set as $X_f=\{x\in\calX~|~f(x)\neq f^*(x)\}$. 
The following \cref{th:tsybakov_lemma}, which bounds the probability of error by a power of the excess risk, is a generalization of the Tsybakov Lemma \cite[Prop.1]{tsybakov2004optimal} for general discrete losses.
\begin{lemma}[Bounding the size of the error set]\label{th:tsybakov_lemma}
If $1/\gamma\in L_p(P_{\calX})$, then
\begin{equation}
    P_{\calX}(X_f)\leq \gamma_p^{\frac{1}{p+1}}(\ERL(f) - \ERL(f^*))^{\frac{p}{p+1}}
\end{equation}
\end{lemma}
\begin{proof}
By the definition of the margin $\gamma(x)$, we have that:
\begin{equation}
    1(f(x)\neq f^*(x))\leq 1/\gamma(x)\Delta \ell(f(x), x)
\end{equation}
By taking the $\frac{p}{p+1}$-th power on both sides, taking the expectation w.r.t $P_{\calX}$ and finally applying H\"older's inequality, we obtain the desired result.
\end{proof}

Before proving \cref{eq:lownoisegeneral}, we will need the following useful \cref{eq:lemma_comparison_ineq} of convex functions.
\begin{lemma}[Property of convex functions]
\label{eq:lemma_comparison_ineq}
    Suppose $h:\Rspace{}\rightarrow\Rspace{}$ is convex and 
    $h(0)=0$. Then, for all $x>0$, $0\leq y\leq x$, 
    \begin{equation}
        h(y)\leq \frac{y}{x}h(x)  \hspace{0.3cm}
        \text{and} \hspace{0.3cm} h(x)/x \text{   is increasing on }(0,\infty).
    \end{equation}
\end{lemma}
\begin{proof}
    Take $\alpha=\frac{y}{x}<1$. The result follows directly by definition of convexity, as
        $h(y) = h((1-\alpha)0 + \alpha x) \leq (1-\alpha)h(0) + \alpha h(x) = \frac{y}{x}h(x)$.
    For the second part, re-arrange the terms in the above inequality.
\end{proof}
The following \cref{eq:lownoisegeneral}, is an adaptation of Thm. 10 of \cite{bartlett2006convexity}, which was specific for binary 0-1 loss, now adapted to the case of general discrete losses.
\begin{theorem}[Improved Calibration]\label{eq:lownoisegeneral}
Suppose that the surrogate $S$ is calibrated with calibration function $H_{L,S}$ (see \cref{eq:calibrationfunctionbayes}) and the $p$-noise condition holds. Then, we have that
\begin{equation}
    H_{L,S,p}(\ERL(d\circ g)-\ERL(f^*)) \leq 
    \calR(g) - \calR(g^*),
\end{equation}
where
\begin{equation}\label{eq:improved_calibration}
    H_{L,S,p}(\varepsilon) = 
    (\gamma_p\varepsilon^p)^{\frac{1}{p+1}} H_{L,S}
    \left(\frac{1}{2}(\gamma_p^{-1}\varepsilon)^{\frac{1}{p+1}}\right).
\end{equation}
Moreover, we have that $H_{L,S,p}(\varepsilon)\geq \gamma_p^{\frac{1}{p+1}} H_{L,S}(\varepsilon/(2\gamma_p^{\frac{1}{p+1}}))$. Hence, $H_{L,S,p}$ never provides a worse rate than $H_{L,S}$.
\end{theorem}

\begin{proof}[Proof. (Of \cref{eq:lownoisegeneral})]
    To ease notation let's denote the excess Bayes risk by $\Delta \ell(z', x)=\ell(z', x)- \ell(f^*(x), x)$.
    
    The intuition of the proof is to split the Bayes excess risk into a part with low noise
    $\Delta \ell(f(x),x)\leq t$ and a part with high noise $\Delta \ell(f(x),x)\geq t$. The first part will be controlled by the $p$-noise assumption and the second part by \cref{eq:calibrationfunctionbayes}.
    \begin{align*}
    \ERL(d\circ g)-\ERL(f^*) &= \Expect_{X}\Delta \ell(f(X), X) \\
    & = \Expect\left\{1(X_f)\Delta \ell(f(X), X)\right\} \\
    & = \Expect\left\{\Delta \ell(f(X), X)1(X_f \cap \{\Delta \ell(f(X),X)\leq t\}\right\} \\
    &+ \Expect\left\{\Delta \ell(f(X), X)1(X_f \cap \{\Delta \ell(f(X),X)\geq t\}\right\} \\
    &= A + B.
    \end{align*}
    \begin{itemize}
        \item \textit{Bounding the error in the region with low noise $A$:}
        \begin{equation}
            A \leq tP_{\calX}(X_f)
            \leq t\gamma_p^{\frac{1}{p+1}}\left(\ERL(d\circ g)-\ERL(f^*)\right)^{\frac{p}{p+1}},
        \end{equation}
        where in the last inequality we have used 
        \cref{th:tsybakov_lemma}.
        \item \textit{Bounding the error in the region with high noise $B$:}
        
            We have that
            \begin{equation}\label{eq:bounding_B}
                 \Delta \ell(f(x), x)1(\Delta \ell(f(x), x)\geq t)
                \leq \frac{t}{H_{L,S}(t)}
                H_{L,S}
                (\Delta\ell(f(x), x))
            \end{equation}
            In the case $\Delta\ell(f(x), x)<t$, inequality in \cref{eq:bounding_B} follows
            from the fact that $H_{L,S}$ is nonnegative.
            For the case $\Delta\ell(f(x), x)>t$, apply
            \cref{eq:lemma_comparison_ineq}
            with $h = H_{L,S}, x= \Delta\ell(f(x), x)$ and $y=t$.

            From \cref{eq:calfunction}, we have that 
            $\Expect \{1(X_f)H_{L,S} (\Delta\ell(f(X), X))\} \leq \calR(g) - \calR(g^*)$. Hence, 
            \begin{equation} B\leq \frac{t}{H_{L,S}(t)}
                (\calR(g) - \calR(g^*))
            \end{equation}
    \end{itemize}

    Putting everything together,
    \begin{equation}
    \ERL(d\circ g)-\ERL(f^*) \leq
    t\gamma_p^{\frac{1}{p+1}}\left(\ERL(d\circ g)-\ERL(f^*)\right)^{\frac{p}{p+1}} +
    \frac{t}{H_{L,S}(t)}
                (\calR(g) - \calR(g^*)),
    \end{equation}
    and hence,
    \begin{equation}
        \left(\frac{\ERL(d\circ g) - \ERL(f^*)}{t}
        -\gamma_p^{\frac{1}{p+1}} \left(\ERL(d\circ g)-\ERL(f^*)\right)^{\frac{p}{p+1}} \right)
        H_{L,S}(t)
        \leq \calR(g) - \calR(g^*).
    \end{equation}
    Choosing $t = \frac{1}{2}\gamma_p^{\frac{-1}{p+1}}
    \left(\ERL(d\circ g)-\ERL(f^*)\right)^{\frac{1}{p+1}}$
    and substituting finally gives \cref{eq:improved_calibration}. The second part of the Theorem follows because $\frac{H_{L,S}(t)}{t}$ is non-decreasing by \cref{eq:lemma_comparison_ineq}.
\end{proof}
Finally, if we apply \cref{eq:lownoisegeneral} to the QS, we get the desired result as 
\cref{th:improvedqscor}.
\begin{corollary}[Improved comparison inequality for QS]\label{th:improvedqscor}
For the QS, we have that 
\begin{equation}\label{eq:improved_comparison_inequality_app}
    \ERL(f) - \ERL(f^*) \leq \gamma_p^{\frac{1}{p+2}}\left(16\|F\|_{\infty}^2(\calR(g)-
    \calR^*)\right)^{\frac{p+1}{p+2}}.
\end{equation}
\end{corollary}
\begin{proof}
Substituting $H_{L,S}(\varepsilon) = \frac{\varepsilon^2}{4\|F\|_{\infty}^2}$ in 
\cref{eq:improved_calibration}, gives that,
\begin{equation}
    H_{L,S,p}= \frac{\varepsilon^{\frac{p+2}{p+1}}}{\gamma_p^{\frac{1}{p+1}}16\|F\|_{\infty}^2}.
\end{equation}
Reversing the relation gives the comparison inequality in \cref{eq:improved_comparison_inequality_app}.
\end{proof}


\section{Multilabel and ranking losses} \label{sec:constant-derivation}

The goal of this section is to derive all of the constants from \cref{table:constants}.

In \cref{sec:prerequisitesanalysis}, we recall the elements that we need in order to derive the constants. In \cref{sec:optimalityapp}, we introduce the main tool from \cite{ramaswamy2016convex}
that we use in order to study the optimality of the QS. Finally, the main bulk is in
\cref{sec:analysislossesapp}, where we analyse each loss separately.

\subsection{Prerequisites. }\label{sec:prerequisitesanalysis}
Remember that the goal here is to study the statistical and computational properties of the QS-estimator $\widehat{f}_n:\calX\xrightarrow{}\calZ$ defined as
\begin{equation}\label{eq:inferenceapp}
    \widehat{f}_n(x) = \argmin_{z\in\calZ}\sum_{i=1}^n\alpha_i(x)L(z, y_i).
\end{equation}
Recall that the statistical complexity is determined by the following quantity,
\begin{equation}\label{eq:low_rank_decompsupp}
    L = FU^\top + c \mathbf{1}.
\end{equation}
where $F=(F_z)_{z\in\calZ}\in\Rspace{|\calZ|\times r},U=(U_y)_{y\in\calY}\in\Rspace{|\calY|\times r}$, $c\in\Rspace{}$ is a scalar and $\mathbf{1} \in \R^{|\calZ|\times|\calY|}$ is the matrix of ones, i.e. $\mathbf{1}_{ij} = 1$ and $r \in \N$.
Here, $F_z$ is the $z$-th row of $F$ and $U_y$ the $y$-th row of $U$.
We denote by $\operatorname{affdim}(L)$ the {\em affine dimension} of the loss $L$, which is defined as the minimum $r$ for which \cref{eq:low_rank_decompsupp} holds.

Recall that the quantity of interest for the statistical complexity is 
\begin{equation}
    \mathsf{A} = \sqrt{r}\|F\|_{\infty}U_{\max}.
\end{equation}

The inference complexity corresponds to the computational complexity of solving \cref{eq:inferenceapp}.
\subsection{On the optimality of the QS. }\label{sec:optimalityapp}

We use results from \cite{ramaswamy2016convex} in order to study the optimality of the dimension of the QS as commented in \cref{rmk:optimality}. We implicitly use the concept of {\em convex calibration dimension of a loss} $L$ (see Def. 10 in \cite{ramaswamy2016convex}), which is defined as the minimum dimension over all consistent convex surrogates w.r.t $L$. In the following \cref{th:optimality} (their Thm. 18), they provide a sufficient condition to lower bound this dimension.

\begin{theorem}(In \cite{ramaswamy2016convex}) \label{th:optimality} Let $L\in\Rspace{\cardZ\times\cardY}$ the loss matrix. If $\exists \Pi\in\operatorname{relint}(\Delta_{\cardY})$, $c\in\Rspace{}$, such that 
$L\Pi=c1$, then there cannot exist any consistent convex surrogate with dimension less than
$\operatorname{affdim}(L)-1$. Here, $\Delta_{\cardY}$ is the simplex of $\cardY$ dimensions and
$\operatorname{relint}(A)$ denotes the relative interior of the set $A$.
\end{theorem}
In particular, \cref{th:optimality} says that if there exists at least one distribution
$\Pi$ at the interior of the simplex for which the Bayes risk is the same for all labels, then one can't hope to be consistent by estimating less than $\operatorname{affdim}(L)-1$ scalar functions. In particular, this means that the QS is essentially optimal over all surrogate methods, in the sense that it estimates $\text{affdim}(L)$ scalar functions.

For each loss, we test the condition given by \cref{th:optimality} to show the optimality (or not) of the Quadratic Surrogate approach.

Note that there exist problems for which you can find consistent surrogates with dimension much smaller than $\text{affdim}(L)$. In ordinal regression, where the discrete labels have a natural order, there exist one dimensional surrogates \cite{pedregosa2017consistency} despite the loss matrix being full rank.

\subsection{Analysis of the losses}\label{sec:analysislossesapp}

\paragraph{Notation.} In the following we denote by $m \in \N$ the number of classes of a multilabel/ranking problem, by $\calP_m$ the power-set of $[m] = \{1,\dots,m\}$ and by $\mathfrak{S}_m$ the set of permutations of $m$-elements. In particular note that in the multilabel problems both the output space $\calZ$ and the observation space $\calY$ are equal to $\calP_m$, while in ranking $\calZ = \mathfrak{S}_m$ and $\calY = [R]^m$, the set of observed relevance scores for the $m$ documents where $R$ is the highest relevance \cite{ravikumar2011ndcg}. Finally we denote by $[v]_j$ the $j$-th element of a vector $v$ and we identify $\calP_m$ with $\{0,1\}^m$, moreover $\sigma(j)$ is the $j$-th element of the permutation $\sigma$, for $\sigma \in \mathfrak{S}_m$, $j \in [m]$.


\subsection*{0-1 loss}
The 0-1 loss is defined as $0$ if the subsets are exactly equal and $1$ otherwise, i.e,
it does not provide any structural information. In this case, $\calY=\calZ=\{0,1\}^m$ and
\begin{equation}
    L(z,y) = 1(z\neq y).
\end{equation}
\begin{itemize}
\item \textbf{Statistical complexity. }
We can decompose it as
\begin{equation}
    F_z = -(1_{[z=z']})_{z'\in\{0,1\}^m},~ U_y = (1_{[y=y']})_{y'\in\{0,1\}^m},~ c=1.
\end{equation}
We have that
\begin{equation}
    r = 2^m, ~\|F\|_{\infty} = 1, ~ U_{\max} = 1.
\end{equation}
Hence,
\begin{equation}
    \mathsf{A} = 2^{m/2}.
\end{equation}

\item \textbf{Inference. } Inference corresponds to
\begin{equation}
    \widehat{f}(x) \in \underset{z\in\calP_m}{\argmax}~\sum_{i|y_i=z}\alpha_i(x),
\end{equation}
which can be done in 
\begin{equation}
    \calO(2^m \wedge n).
\end{equation}

\item \textbf{Optimality of $r$. }
    Taking $\Pi_y=1/{2^m}$ for every $y\in\calY$ and applying \cref{th:optimality},
    one has that $\operatorname{affdim}(L)=2^m$ is optimal.
\end{itemize}


\subsection*{Block 0-1 loss}
Assume that the prediction space $\calP_m$ is partitioned into $b$
regions $\calP_m=\sqcup_{j=1}^bB_{j}$. The block 0-1 loss is defined as $0$ if the subsets belong to the same region and $1$ otherwise. In this case, $\calY=\calZ=\{0,1\}^m$ and
\begin{equation}
    L(z,y) = 1(z\in B_{j}, y\notin B_{j}, \text{for some } j\in[b]).
\end{equation}
\begin{itemize}
\item \textbf{Statistical complexity. }
We can decompose it as
\begin{equation}
     F_z = -(1_{[z\in B_j]})_{j=1}^b,~ U_y = (1_{[y\in B_{j}]})_{{j}=1}^b,~ c=1.
\end{equation}
We have that
\begin{equation}
    r = b, ~\|F\|_{\infty} = 1, ~ U_{\max} = 1.
\end{equation} 
Hence,
\begin{equation}
    \mathsf{A} = \sqrt{b}.
\end{equation}
\item \textbf{Inference. } Inference corresponds to
\begin{equation}
   \widehat{f}(x) \in \underset{1\leq j\leq b}{\argmax}~ \sum_{i|y_i\in B_j}\alpha_i(x),
\end{equation}
which can be done in 
\begin{equation}
    \calO(b)
\end{equation}

\item \textbf{Optimality of $r$. }
Taking $\Pi_y=\frac{1}{b|B(y)|}$, where $B(y)$ is the partition where $y\in\calY$ belongs to and applying  \cref{th:optimality},
 one has that $\operatorname{affdim}(L)$ is optimal.
\end{itemize}


\subsection*{Hamming}
The Hamming loss counts the average number of classes that
disagree. In this case, $\calY=\calZ=\{0,1\}^m$ and
\begin{equation}
    L(z,y) = \frac{1}{m}\sum_{j=1}^m1([z]_j\neq [y]_j).
\end{equation}
\begin{itemize}
\item \textbf{Statistical complexity. }
If we define $s_j(y) = 2[y]_j-1$, we can re-write the Hamming loss as
\begin{align*}
    L(z,y) 
    = \frac{1}{m}\sum_{j=1}^m\left(\frac{1 - s_{j}(z)s_{j}(y)}{2}\right)
    = \frac{1}{2m} - \frac{1}{2m}\sum_{j=1}^ms_{j}(z)s_{j}(y).
\end{align*}
This implies that
\begin{equation}
    F_z = -\frac{1}{2m}(s_j(z))_{j=1}^{m},~ U_y = (s_j(y))_{j=1}^{m},~ c=\frac{1}{2m}.
\end{equation}
We have that
\begin{equation}
    \|F\|_{\infty} = \frac{1}{2\sqrt{m}}, ~ U_{\max} = 1.
\end{equation}
Hence,
\begin{equation}
    \mathsf{A} = \frac{1}{2}.
\end{equation}

\item \textbf{Inference. } 
Inference corresponds to 
\begin{equation}
    \widehat{f}_j(x) = \left(\frac{\text{sign}\left(\widehat{g}_j(x)\right)+1}{2}\right), \hspace{0.5cm} \text{where}
    \hspace{0.5cm}
    \widehat{g}_j(x) = \sum_{i=1}^ns_j(y_i)\alpha_i(x),
\end{equation}
which can be done in 
\begin{equation}
    \calO(m).
\end{equation}

\item \textbf{Optimality of $r$. } Taking $\Pi_y=1/{2^m}$ for every $y\in\calY$ and applying \cref{th:optimality},
    one has that $\operatorname{affdim}(L)=m$ is optimal.
\end{itemize}


\subsection*{Prec@k}
Prec@k (Precision at k) measures the average number of elements in the
predicted $k$-set that also belong to the ground truth.  In this case, the prediction space is $\calZ=\calP_{m,k}$, i.e, subsets of $[m]$ of size $k$, and $\calY=\calP_{m}$. 
\begin{equation}
    L(z,y) = 1 - \frac{|y\cap z|}{k} = 1 - \frac{1}{k}\sum_{j=1}^m[z]_j[y]_j.
\end{equation}
\begin{itemize}
\item \textbf{Statistical complexity. }
We have that $ r = m, F_z = -\frac{1}{k}([z]_j)_{j=1}^{m},~ U_y = ([y]_j)_{j=1}^m,~ c=1,
\|F\|_{\infty} = \frac{1}{\sqrt{k}}, ~ U_{\max} = 1 $. Hence,
\begin{equation}
    \mathsf{A} = \sqrt{\frac{m}{k}}.
\end{equation}

\item \textbf{Inference. } Inference corresponds to
\begin{equation}
   \widehat{f}(x) \in \underset{z\in\calP_{m,k}}{\text{arg top}_k}\left(\left(
   ~\sum_{i|[y_i]_j=1}\alpha_i(x)\right)_{j=1}^m\right),
\end{equation}
which can be done in 
\begin{equation}
    \calO(m\log k).
\end{equation}

\item \textbf{Optimality of $r$. } Taking $\Pi_y=1/{2^m}$ for every $y\in\calY$ and applying \cref{th:optimality},
    one has that $\operatorname{affdim}(L)=m$ is optimal.
\end{itemize}


\subsection*{F-score}
The F-score is defined as the harmonic mean of precision and recall. In this case
$\calZ=\calY=\calP_m$ and
\begin{equation}
    L(z,y) = 1 - 2\frac{|z\cap y|}{|z| + |y|},
\end{equation}
where we treat the case $y=0$ as follows:
\begin{equation}\label{eq:fscore}
   2\frac{|z\cap y|}{|z| + |y|} = \left\{ \begin{array}{ll}
         2\sum_{j=1}^m\sum_{\ell=0}^{m}\frac{[z]_j}{\ell + |z|}1([y]_j=1, |y|=\ell) & y\neq 0 \\
         1(z=0) & y=0
    \end{array}\right. .
\end{equation}
Let's define the matrix $P(x)\in\Rspace{m\times m}$ and $p_0(x)\in\Rspace{}$ as,
\begin{equation}
    P_{j\ell}(x) = P([Y]_j=1,|Y|=\ell|X=x),\hspace{0.5cm} p_0(x)=P(Y=0|X=x).
\end{equation}
Then, the Bayes risk reads
\begin{equation}
    \ell(z, x) = \left\{ \begin{array}{ll}
         2\sum_{j=1}^m\sum_{\ell=0}^{m}\frac{[z]_j}{\ell + |z|}P_{j\ell}(x) & y\neq 0 \\
         p_0(x) & y=0
    \end{array}\right. .
\end{equation}
Hence, for every $x$, one needs no more than $r=m^2+1$ parameters to compute the F-score Bayes risk.

We have the following \cref{th:fscorealgo}.
\begin{lemma}\label{th:fscorealgo}
Given the matrix $P(x)\in\Rspace{m\times m}$ and the scalar $p_0(x)$, inference can be
performed through the following two-step procedure:
\begin{enumerate}
    \item Compute the matrix $A(x)\in\Rspace{m\times m}$:
    \begin{equation}\label{eq:relationAP}
        A_{jk}(x) = \sum_{\ell=0}^m\frac{P_{j\ell}(x)}{\ell + k}
    \end{equation}
    This is a matrix-by-matrix multiplication that takes $\calO(m^3)$.
    \item From $A(x)$ and $p_0(x)$, the prediction $f(x)$ can be computed in
    $\calO(m^2)$ through an iterated maximization procedure.
\end{enumerate}
\end{lemma}
\begin{proof}
Suppose we have already computed $A(x)\in\Rspace{m\times m}$ and $p_0(x)\in\Rspace{}$.
Now, we perform the following $m$ maximizations:
\begin{equation}
    f^{(k)}(x) = \argmax_{z\in\calP_{m,k}}A_{\cdot, k}^T(x)z,\hspace{0.5cm}
    \text{for} ~ k=1,\ldots,m.
\end{equation}
Then, $f^*(x)$ is computed by taking the maximum over the $f^{(k)}(x)$'s together with $p_0(x)$, which corresponds to $z=0$.
\end{proof}

\begin{itemize}
\item \textbf{Statistical complexity. }

Note that depending on whether we approximate $P$ or directly $A$, we have different computational complexities. In particular, if the surrogate approximates directly $A$, then it avoids the operation \cref{eq:relationAP}. As the estimator is the same, the statistical complexity is the minimum of both.

\textit{Decomposition 1.} 
Estimating $P(x)$, corresponds to the following decomposition:
\begin{align*}
    F_{z, m(\ell-1)+j} = -\left(\frac{1([z]_j=1)}{|z| + \ell}\right), \hspace{0.5cm} 1\leq j,\ell\leq m \\
    U_{y, m(\ell-1)+j} = 1([y]_j=1, |y|=\ell), \hspace{0.5cm} 1\leq j,\ell\leq m
\end{align*} 
and $F_{z,m^2+1}=1(z=0),U_{y,m^2+1}=1(y=0)$. 
In this case,
\begin{equation}
    r=m^2+1, ~\frac{1}{2} \leq \|F\|_{\infty} \leq 1, ~ U_{\max} = 1.
\end{equation}
Hence,
\begin{equation}
    \mathsf{A}_1 \leq \sqrt{m^2+1} \leq \sqrt{2}m
\end{equation}

\textit{Decomposition 2.}  Estimating $A(x)$, corresponds to the following decomposition:
\begin{align*}
    F_{z, m(\ell-1)+j} =  -1([z]_j=1, |z|=\ell) , \hspace{0.5cm} 1\leq j,\ell\leq m \\
    U_{y, m(\ell-1)+j} =\left(\frac{1([y]_j=1)}{|y| + \ell}\right), \hspace{0.5cm} 1\leq j,\ell\leq m
\end{align*} 
and $F_{z,m^2+1}=1(z=0),U_{y,m^2+1}=1(y=0)$. 

In this case,
\begin{equation}
    r=m^2+1, ~\|F\|_{\infty} = \sqrt{m}, ~ U_{\max} = 1.
\end{equation}
Hence,
\begin{equation}
    \mathsf{A}_2 = \sqrt{m(m^2+1)}\leq m\sqrt{2m}.
\end{equation}
We take $\mathsf{A}=\min(\mathsf{A}_1,\mathsf{A}_2)$, hence,
\begin{equation}
    \mathsf{A} \leq \sqrt{2}m.
\end{equation}
\item \textbf{Inference. } 
The quadratic surrogate approximates $A(x)$ and $P(x)$ as:
\begin{equation}
    \widehat{P}_{j\ell}(x) = \sum_{i|[y_i]_j=1, |y_i|=\ell}\alpha_i(x),
    \hspace{0.3cm} 
    \widehat{A}_{jk}(x) = \sum_{\ell=0}^m\frac{\widehat{P}_{j\ell}(x)}{\ell + k}, ~
    \widehat{p}_0(x) = \sum_{i|y_i=0}\alpha_i(x).
\end{equation}

If we use \textit{Decomposition 1}, i.e, $\widehat{g}(x)=(\widehat{P}(x),\widehat{p}_0(x))$, then
we have cubic inference,
\begin{equation}
    \calO(m^3).
\end{equation}

If we use \textit{Decomposition 2}, i.e, $\widehat{g}(x)=(\widehat{A}(x),\widehat{p}_0(x))$, then
we have quadratic inference,
\begin{equation}
    \calO(m^2).
\end{equation}

\item \textbf{Optimality of $r$. }
We can't say anything about the potential existence of a convex calibrated surrogate with smaller dimension than $\text{affdim}(L)-1$. This is because the sufficient condition from
Theorem 18 of \cite{ramaswamy2016convex} does not hold for any $\Pi$ even for $m=2$. 
\end{itemize}


\subsection*{NDCG-type}

Let $\calZ=\mathfrak{S}_m$ be the set of permutations of $m$ elements and $\calY=\{1,\ldots,R\}^m=[R]^m$ the space of relevance scores for $m$ documents. Let
the \textit{gain} $G:\Rspace{}\rightarrow\Rspace{}$ be an increasing function and the \textit{discount} vector $D=(D_j)_{j=1}^m$ be a coordinate-wise decreasing vector. NDCG-type losses are defined as the normalized discounted sum of the gain of the relevance scores ordered by the predicted permutation: 
\begin{equation}\label{eq:ndcgtype}
    L(\sigma, r) = 1 - \frac{1}{N(r)}\sum_{j=1}^mG([r]_j)D_{\sigma(j)}
\end{equation}
where $N(r)=\max_{\sigma\in\mathfrak{S}_m}\sum_{j=1}^mG([r]_j)D_{\sigma(j)}$ is the normalizer.
The discount is performed in order to give more importance to the relevance of the top ranked elements.

\begin{itemize}
\item \textbf{Statistical complexity. }

Note that looking at \cref{eq:ndcgtype} we can directly write that
 $r=m$ and
$F_\sigma = -(D_{\sigma(j)})_{j=1}^m, U_r = \left(\frac{G([r]_j)}{N(r)}\right)_{j=1}^m, c=1$.

It follows that,
\begin{equation}
    \|F\|_{\infty} = \sqrt{\sum_{j=1}^mD_j^2}, \quad
    U_{\max} = G_{\max} D_{\max},
\end{equation}
hence, 
\begin{equation}
 \mathsf{A} = \sqrt{m}G_{\max} D_{\max}\sqrt{\sum_{j=1}^mD_j^2}.   
\end{equation}

\item \textbf{Inference. }

The inference corresponds to,
\begin{equation}
\widehat{f}(x) = \operatorname{argsort}_{\sigma\in\mathfrak{S}_m}(v),
\hspace{0.5cm}\text{where} ~
    v_j = \sum_{i=1}^n\frac{G([r_i]_j)\alpha_i(x)}{N(r_i)},
\end{equation}
which can be done in 
\begin{equation}
    \calO(m\log m)
\end{equation} operations.

\item \textbf{Optimality of $r$. } Optimal. As Hamming, the barycenter of the simplex satisfies \cref{th:optimality}.
\end{itemize}

\paragraph{Normalized Discounted Cumulative Gain (NDCG)}
This is the most widely used configuration, in this case, $G(t)=2^t-1$ and $D_j = \frac{1}{\log(j+1)}$.
We have that  $\|D\|_2 \sim \left(\int_{2}^{m}\frac{1}{\log^2(t)}dt\right)^{1/2}\sim \sqrt{\frac{m}{\log m}}$. And hence, 
\begin{equation}
    \mathsf{A} \leq c G_{\max}\frac{m}{\sqrt{\log m}}.
\end{equation}

\paragraph{Expected Rank Utility (ERU)}
In this case, $G(t)=\max(t-\bar{r})$ and $D_j = 2^{1-j}$, where $\bar{r}$ corresponds to a neutral score.
We have that  $\|D\|_2 \leq \frac{2}{\sqrt{3}}$,
\begin{equation}
    A \leq \frac{2}{\sqrt{3}}G_{\max}\sqrt{m}.
\end{equation}

The QS-estimator estimates the marginals of the normalized relevance scores and sorts the estimates at inference. As it was shown in \cite{ravikumar2011ndcg}, in order to be consistent for NDCG, one has to estimate the {\em normalized} relevance scores and not the {\em unnormalized} ones as one would do at the first place. In particular, the QS-estimator for the NDCG that follows directly from our framework corresponds exactly to their proposed consistent algorithm.

Due to the discount factor, the statistical complexity grows with the number of elements to sort. In particular, faster the decay is, more samples you need to optimize the corresponding loss.
This is shown in the two examples we have shown, where the NDCG is statistically easier to optimize than the ERU.


\subsection*{Pairwise Disagreement (PD)}
The pairwise disagreement computes the cost associated to a given permutation in terms of pairwise comparisons using binary relevance scores. In this case, $\calZ=\mathfrak{S}_m$, $\calY=[0,1]^m=\calP_m$, and,
\begin{equation}
L(\sigma,y) = \frac{1}{N(y)}\sum_{j=1}^{m}\sum_{\ell\neq j}
1([y]_j<[y]_{\ell})1(\sigma(j)>\sigma(\ell)),
\end{equation}
where $N(y)=\sup_{\sigma\in\mathfrak{S}_m}\sum_{j=1}^{m}\sum_{\ell\neq j}
1([y]_j<[y]_{\ell})1(\sigma(j)>\sigma(\ell))=|y|(m-|y|)$ is a normalizer.
    
\begin{itemize}
\item \textbf{Statistical complexity. }
Note that we can re-write 
\begin{equation}
    1([y]_j<[y]_\ell) = \frac{\text{sign}([y]_\ell - [y]_j) + 1}{2},
    \hspace{0.5cm} 1(\sigma(j)>\sigma(\ell)) = \frac{\text{sign}(\sigma(j) - \sigma(\ell)) + 1}{2}.
\end{equation}
Hence,
\begin{equation}
    L(\sigma,y) = \frac{1}{4} + \frac{1}{4N(y)}\sum_{j=1}^m\sum_{\ell\neq j}\text{sign}([y]_\ell - [y]_j) \text{sign}(\sigma(j) - \sigma(\ell))
\end{equation}
Note that $F_\sigma = 1/4(\text{sign}(\sigma(j) - \sigma(\ell)))_{j,\ell=1}^m$ and
$U_y = (\frac{\text{sign}([y]_\ell - [y]_j)}{N(y)})_{j,\ell=1}^m$ are anti-symmetric matrices. Hence, they can be described with $m(m-1)/2$ numbers. We can then consider $F_\sigma$ and $U_y$ as vectors of $m(m-1)/2$
coordinates.

This implies that $ r = m(m-1)/2 ,~ c=1/4, \|F\|_{\infty} = 1/4\sqrt{m(m-1)/2},
    ~ U_{\max} = \frac{2}{m-1} $. Hence,
    \begin{equation}
        \mathsf{A} = \frac{m}{4}
    \end{equation}
    
\item \textbf{Inference. } In this case, the optimization problem reads
\begin{equation}
    \widehat{f}(x) \in \argmin_{\sigma\in\mathfrak{S}_m}\sum_{j=1}^{m}\sum_{\ell\neq j}
    \gamma_{j\ell}(x)1(\sigma(j)>\sigma(\ell)),
\end{equation}
with 
\begin{equation}
    \gamma_{j\ell}(x) = \sum_{i|[y_i]_j<[y_i]_{\ell}}\frac{\alpha_i(x)}{N(y_i)}.
\end{equation}
This is precisely a Minimum Weight Feedback Arcset (MWFAS) problem with associated directed graph
having weights $\gamma_{j\ell}(x)$. This problem is known to be NP-Hard.

\item \textbf{Optimality of $r$. } Optimal. See Corollary 19 and Proposition 20 from
\cite{ramaswamy2016convex}.
\end{itemize}

As it was shown in \cite{calauzenes2012non}, there is no hope of devising a consistent convex surrogate method which is based on sorting an estimated vector of relevance scores. In particular, one needs to estimate $\frac{m(m-1)}{2}$ scalar functions corresponding to the weights of a graph between the classes. Although estimating the graph structure is statistically feasible, inference corresponds to finding a directed acyclic graph (DAG) with minimum cost. This is equivalent to the Minimum Weight Feedback Arcset Problem (MWFAS), which is known to be NP-Hard. 
Consequently, one can state that, unless $P=NP$, there does not exist any polynomial surrogate-based consistent algorithm for the PD loss. If it existed, one could solve the Bayes risk minimization problem, i.e., MFWAS,  to $\varepsilon$-accuracy in $\text{poly}(\frac{1}{\varepsilon})$.


\subsection*{Mean Average Precision (MAP)}
The mean average precision (MAP) is a widely used ranking measure in
information retrieval. The precision associated to a relevant document $j$ ($[y]_j=1$) ranked at position
$\sigma(j)$ is the Precision at $\sigma(j)$ of the $\sigma(j)$ retrieved documents ranked before (and including),
$j$. In this case, $\calZ=\mathfrak{S}_m$ and $\calY=[0,1]^m=\calP_m$. The mean average precision corresponds to the mean over all relevant documents in $y$. Hence, MAP has the following form:
\begin{equation}
    L(\sigma, y)
    = 1 - \frac{1}{|y|}\sum_{j|[y]_j=1}\frac{1}{\sigma(j)}
    \sum_{\ell=1}^{\sigma(j)}[y]_{\sigma^{-1}(\ell)}.
\end{equation}
Note that it can be re-written as 
\begin{align*}
    L(\sigma, y) & = 
    1 - \frac{1}{|y|}\sum_{j=1}^m\frac{[y]_j}{\sigma(j)} 
    \sum_{\ell=1}^{\sigma(j)}[y]_{\sigma^{-1}(\ell)} \\
    &= 1 - \frac{1}{|y|}\sum_{j=1}^m
    \sum_{\ell=1}^j\frac{[y]_{\sigma^{-1}(\ell)}[y]_{\sigma^{-1}(j)}}{j} \\
    &= 1 - \frac{1}{|y|}\sum_{j=1}^m
    \sum_{\ell=1}^j\frac{[y]_{\ell}[y]_j}{\max(\sigma(j), \sigma(\ell))}. \\
    \end{align*}
    
\begin{itemize}
    \item \textbf{Statistical complexity. }
We have that $ r = \frac{m(m+1)}{2}, 
F_\sigma = \left(\max(\sigma(j), \sigma(\ell))^{-1}\right)_{j\geq \ell},~
    U_y = -\left(\frac{[y]_j[y]_{\ell}}{|y|}\right)_{j\geq\ell},~ c=1, \|F\|_{\infty} \leq \sqrt{\log (m+1)},
    ~ U_{\max} = 1/2 $.
    Hence, 
    \begin{equation}
        A = \frac{1}{2}m\sqrt{\log (m+1)}
    \end{equation}
    
\item \textbf{Computational complexity. }
The inference problem reads
\begin{equation}
    \widehat{f}(x) = \underset{\sigma\in \mathfrak{S}_m}{\argmax} 
   \sum_{j=1}^m
    \sum_{\ell=1}^j\frac{1}{\max(\sigma(j), \sigma(\ell))}
    \sum_{i|[y_i]_{j}[y_i]_{\ell}=1}\frac{\alpha_i(x)}{|y_i|}
\end{equation}
Denote by 
 \begin{equation}
    W_{j\ell}=\left\{\begin{array}{cc} \sum_{i|[y_i]_{j}[y_i]_{\ell}=1}\frac{\alpha_i(x)}{|y_i|}
    & j\geq \ell \\ 0 & \text{otherwise} \end{array}\right.
    ,~
    D_{j\ell}=\left\{\begin{array}{cc} \max(j, \ell)^{-1}
    & j\geq \ell \\ 0 & \text{otherwise} \end{array}\right.
\end{equation}
    We have that,
    \begin{equation}
       \widehat{f}(x) = \underset{\sigma\in \mathfrak{S}_m}{\argmax}
    \sum_{j,\ell=1}^m W_{j\ell}D_{\sigma(j)\sigma(\ell)} \equiv  \underset{P\in \Pi_m}{\argmax}~\operatorname{Tr}(W^TPDP^T),
    \end{equation}
    where $\Pi_m$ is the set of permutation matrices of size $m$. This is an instance of the Quadratic Assignment Problem (QAP).
     
\item \textbf{Optimality of $r$. }Optimal. See Corollary 19 and Proposition 21 from
\cite{ramaswamy2016convex}.
     
\end{itemize}

As for PD, inference for MAP corresponds to a NP-Hard problem, more specifically, to an instance of the Quadratic Assignment Problem (QAP). Consequently, one can conclude analogously as for the PD loss, i.e., that no efficient and consistent surrogate algorithm exists for MAP.
\end{document}